\documentclass{article}

\PassOptionsToPackage{numbers, compress}{natbib}

\usepackage[preprint]{neurips_2022}

\usepackage[utf8]{inputenc} 
\usepackage[T1]{fontenc}    
\usepackage{hyperref}       
\usepackage{url}            
\usepackage{booktabs}       
\usepackage{amsfonts}       
\usepackage{nicefrac}       
\usepackage{microtype}      
\usepackage{xcolor}         
\usepackage{amsthm, amssymb, amsfonts, latexsym, mathtools}
\usepackage[ruled,vlined]{algorithm2e}
\usepackage{algorithmic}
\usepackage{comment}
\usepackage{here}
\usepackage{subfigure}
\usepackage{wrapfig}

\newtheorem{theorem}{Theorem}
\newtheorem{lemma}{Lemma}
\newtheorem{corollary}{Corollary}
\newtheorem{remark}{Remark}
\newtheorem{definition}{Definition}
\newtheorem{assumption}{Assumption}
\newtheorem{example}{Example}

\title{Communication Compression for Decentralized Learning with Operator Splitting Methods}

\author{%
  Yuki Takezawa \\
  Kyoto University and RIKEN AIP\\
  \texttt{yuki-takezawa@ml.ist.i.kyoto-u.ac.jp} \\
  \And
  Kenta Niwa \\
  NTT Communication Science Laboratories \\
  \texttt{kenta.niwa.bk@hco.ntt.co.jp} \\
  \And
  Makoto Yamada \\
  Kyoto University and RIKEN AIP \\
  \texttt{myamada@i.kyoto-u.ac.jp}
}

\begin{document}

\maketitle

\begin{abstract}
In decentralized learning, operator splitting methods using a primal-dual formulation (e.g., the Edge-Consensus Learning (ECL))
has been shown to be robust to heterogeneous data and has attracted significant attention in recent years.
However, in the ECL, a node needs to exchange dual variables with its neighbors.
These exchanges incur significant communication costs.
For the Gossip-based algorithms, many compression methods have been proposed, but these Gossip-based algorithm do not perform well when the data distribution held by each node is statistically heterogeneous. In this work, we propose the novel framework of the compression methods for the ECL,
called the Communication Compressed ECL (C-ECL).
Specifically, we reformulate the update formulas of the ECL,
and propose to compress the update values of the dual variables.
We demonstrate experimentally that the C-ECL can achieve a nearly equivalent performance with fewer parameter exchanges than the ECL.
Moreover, we demonstrate that the C-ECL is more robust to heterogeneous data than the Gossip-based algorithms.
\end{abstract}

\section{Introduction}
In recent years, neural networks have shown promising results in various fields, 
including image processing \cite{chen2020simple,dosovitskiy2021an}
and natural language processing \cite{vaswani2017attention,devlin2019bert},
and have thus attracted considerable attention.
To train a neural network, we generally need to collect a large number of training data.
Owning to the use of crowdsourcing services, it is now easy to collect a large number of annotated images and texts.
However, because of privacy concerns, it is difficult to collect a large number of \emph{personal} data, 
such as medical data including gene expression and medical images, on a single server. 
In such cases, decentralized learning, which aims to train a model without sharing the training data among servers, is a powerful tool. Decentralized learning was originally studied to train large-scale models in parallel, 
and because it allows us to train models without aggregating the training data, 
it has recently attracted significant attention from the perspective of privacy preservation.

One of the most widely used algorithms for decentralized learning is the Gossip-based algorithm \citep{boyd2006randomized,lian2017can}.
In the Gossip-based algorithm, each node (i.e., server) updates the model parameters using its own gradient,
exchanges model parameters with its neighbors,
and then takes the average value to reach a consensus.
The Gossip-based algorithm is a simple yet effective approach.
When the distribution of the data subset held by each node is statistically homogeneous,
the Gossip-based algorithm can perform as well as the vanilla SGD, which trains the model on a single node using all of the training data.
However, when the data distribution of each node is statistically heterogeneous 
(e.g., each node has only images of some classes and not others),
\textit{the client-drift} \citep{karimireddy2020scaffold} occurs, and the Gossip-based algorithms do not perform well \cite{tang2018d2,vogels2021relaysum}.

Recently, the operator splitting method using the primal-dual formulations, called the Edge-Consensus Learning (ECL) \citep{niwa2020edge}, has been proposed.
The primal-dual algorithm, including the Alternating Direction Method of Multipliers (ADMM) \cite{boyd2011distributed} and the Primal-Dual Method of Multipliers (PDMM) \cite{zhang2018distributed}, can be applied in decentralized learning by representing the model consensus as the linear constraints.
It has recently been shown that the ADMM and PDMM for decentralized learning can be derived by solving the dual problem using operator splitting methods \citep{sherson2019derivation} (e.g., the Douglas-Rachford splitting \cite{douglas1955numerical} and the Peaceman-Rachford splitting \cite{peaceman1955numerical}).
In addition, \citet{niwa2020edge,niwa2021asynchronous} applied these operator splitting methods to the neural networks, 
named them the Edge-Consensus Learning (ECL).
Then, they showed that the ECL can be interpreted as a variance reduction method and is robust to heterogeneous data.

However, for both the Gossip-based algorithm and the ECL,
each node needs to exchange the model parameters and/or dual variables with its neighbors,
and such exchanges incur significant communication costs.
Recently, in the Gossip-based algorithm, many studies have proposed methods for compressing the exchange of parameters \cite{tang2018communication,koloskova2019decentralized,lu2020monniqua,vogels2020powergossip}.
Then, they showed that these compression methods can train a model 
with fewer parameter exchanges than the uncompressed Gossip-based algorithm.
However, these compression methods are based on the Gossip algorithms
and do not perform well when the data distribution of each node is statistically heterogeneous.

In this work, we propose the novel framework for the compression methods applied to the ECL,
which we refer to as \textit{the Communication Compressed ECL (C-ECL)}.
Specifically, we reformulate the update formulas of the ECL,
and propose to compress the update values of the dual variables.
Theoretically, we analyze how our proposed compression affects the convergence rate of the ECL
and show that the C-ECL converges linearly to the optimal solution as well as the ECL.
Experimentally, we show that the C-ECL can achieve almost the same accuracy as the ECL with fewer parameter exchanges.
Furthermore, 
the experimental results show that the C-ECL is more robust to heterogeneous data than the Gossip-based algorithm,
and when the data distribution of each node is statistically heterogeneous,
the C-ECL can outperform the uncompressed Gossip-based algorithm, in terms of both the accuracy and the communication costs.

\textbf{Notation:}
In this work,
$\|\cdot\|$ denotes the L2 norm,
$\mathbf{0}$ denotes a vector with all zeros,
and $\mathbf{I}$ denotes an identity matrix.

\section{Related Work}
In this section, we briefly introduce the problem setting of decentralized learning,
and then introduce the Gossip-based algorithm and the ECL.

\subsection{Decentralized Learning}
Let $G=(\mathcal{V}, \mathcal{E})$ be an undirected connected graph that represents the network topology 
where $\mathcal{V}$ denotes the set of nodes and $\mathcal{E}$ denotes the set of edges.
For simplicity, we denote $\mathcal{V}$ as the set of integers $\{1,2,\ldots, |\mathcal{V}|\}$. 
We denote the set of neighbors of the node $i$ as $\mathcal{N}_i=\{j\in \mathcal{V} | (i, j) \in \mathcal{E}\}$. 
The goal of decentralized learning is formulated as follows:
\begin{align}
\label{eq:decentralized_learning}
    \inf_{\mathbf{w}} \; \sum_{i\in \mathcal{V}} f_i(\mathbf{w}), \;\; f_i(\mathbf{w}) \coloneqq\mathbb{E}_{\zeta_i \sim \mathcal{D}_i} [F(\mathbf{w}; \zeta_i)],
\end{align}
where $\mathbf{w}\in\mathbb{R}^d$ is the model parameter, $F$ is the loss function, $\mathcal{D}_i$ represents the data held by the node $i$,
$\zeta_i$ is the data sample from $\mathcal{D}_i$,
and $f_i$ is the loss function of the node $i$.

\subsection{Gossip-based Algorithm}
\label{sec:gossip}
A widely used approach for decentralized learning is the Gossip-based algorithm (e.g. D-PSGD \cite{lian2017can}). 
In the Gossip-based algorithm, each node $i$ computes its own gradient $\nabla f_i$,
exchanges parameters with its neighbors,
and then takes their average. 
Although, the Gossip-based algorithm is a simple and effective approach,
there are two main issues:
high communication costs and sensitivity to the heterogeneity of the data distribution.

In the Gossip-based algorithm,
the node needs to receive the model parameter from its neighbors.
Because the number of parameters in a neural network is large,
the exchanges of the model parameters incur huge communication costs.
To reduce these communication costs of the Gossip-based algorithm, 
many methods that compress the exchanging parameters by using sparsification, quantization, and low-rank approximation have been proposed \cite{tang2018communication,koloskova2019decentralized,lu2020monniqua,Koloskova2020Decentralized,vogels2020powergossip}. 
Then, it was shown that these compression methods can achieve almost the same accuracy as the uncompressed Gossip-based algorithm
with fewer parameter exchanges.

The second issue is that the Gossip-based algorithm is sensitive to the heterogeneity of the data distribution.
When the data distribution of each node is statistically heterogeneous, 
because the optimal solution of the loss function $\sum_i f_i$
and the optimal solution of the loss function of each node $f_i$ are far from each other,
the Gossip-based algorithms do not perform well \cite{niwa2020edge,vogels2021relaysum}.
To address heterogeneous data in the Gossip-based algorithm, 
\citet{tang2018d2} and \citet{xin2020variance} applied variance reduction methods \cite{defazio2014saga,johnson2013accelerating} 
to the Gossip-based algorithm.
\citet{lorenzo2016next} proposed gradient tracking algorithms,
and \citet{vogels2021relaysum} recently proposed the RelaySum.

\subsection{Edge-Consensus Learning}
\label{sec:primal_dual}
In this section, we briefly introduce the Edge-Consensus Learning (ECL) \cite{niwa2020edge}.
By reformulating Eq. \eqref{eq:decentralized_learning}, the primal problem can be defined as follows:
\begin{align}
\label{eq:ecl_primal}
    \inf_{\{\mathbf{w}_i\}_i} \sum_{i\in\mathcal{V}} f_i(\mathbf{w}_i) \;\;
    \text{s.t.} \;\; \mathbf{A}_{i|j} \mathbf{w}_i + \mathbf{A}_{j|i} \mathbf{w}_j = \mathbf{0}, \; (\forall(i,j) \in \mathcal{E}),
\end{align}
where $\mathbf{A}_{i|j}=\mathbf{I}$ when $j\in\mathcal{N}_i$ and $i<j$, 
and $\mathbf{A}_{i|j}=-\mathbf{I}$ when $j\in\mathcal{N}_i$ and $i>j$.
By contrast, the Gossip-based algorithm explicitly computes the average at each round, 
whereas the primal problem of Eq. \eqref{eq:ecl_primal} represents the consensus based on the linear constraints.
Subsequently, by solving the the dual problem of Eq. \eqref{eq:ecl_primal} using the Douglas-Rachford splitting \citep{douglas1955numerical},
the update formulas can be derived as follows \citep{sherson2019derivation}:
\begin{align}
    \label{eq:update_w}
    \mathbf{w}^{(r+1)}_i &= \text{argmin}_{\mathbf{w}_i} \{ f_i(\mathbf{w}_i)
    + \frac{\alpha}{2} \sum_{j\in\mathcal{N}_i} {\left\| \mathbf{A}_{i|j} \mathbf{w}_i - \frac{1}{\alpha}\mathbf{z}^{(r)}_{i|j} \right\|}^2 \}, \\
    \label{eq:update_y}
    \mathbf{y}^{(r+1)}_{i|j} &= \mathbf{z}^{(r)}_{i|j} - 2 \alpha \mathbf{A}_{i|j} \mathbf{w}^{(r+1)}_i, \\
    \label{eq:update_z}
    \mathbf{z}^{(r+1)}_{i|j} &= (1 - \theta) \mathbf{z}^{(r)}_{i|j} + \theta \mathbf{y}^{(r+1)}_{j|i},
\end{align}
where $\theta \in (0, 1]$ and $\alpha>0$ are hyperparameters,
and $\mathbf{y}_{i|j}, \mathbf{z}_{i|j} \in \mathbb{R}^d$ are dual variables.
We present detailed derivation of these update formulas in Sec. \ref{sec:derivation_of_ecl}.
When $\theta=1$, the Douglas-Rachford splitting is specifically called the Peaceman-Rachford splitting \citep{peaceman1955numerical}.
When $f_i$ is non-convex (e.g., a loss function of a neural network), Eq. \eqref{eq:update_w} can not be generally solved.
Then, \citet{niwa2020edge} proposed the Edge-Consensus Learning (ECL), which approximately solves Eq. \eqref{eq:update_w} as follows: 
\begin{align}
    \mathbf{w}^{(r+1)}_i &= \text{argmin}_{\mathbf{w}_i} \{ \langle \mathbf{w}_i, \nabla f_i(\mathbf{w}^{(r)}_i) \rangle
    + \frac{1}{2\eta} \left\| \mathbf{w}_i - \mathbf{w}^{(r)}_i \right\|^2
    + \frac{\alpha}{2} \sum_{j\in\mathcal{N}_i} {\left\| \mathbf{A}_{i|j} \mathbf{w}_i - \frac{1}{\alpha}\mathbf{z}^{(r)}_{i|j} \right\|}^2 \},
\end{align}
where $\eta>0$ corresponds to the learning rate.
Then, \citet{niwa2021asynchronous} showed that the ECL can be interpreted as a stochastic variance reduction method
and demonstrated that the ECL is more robust to heterogeneous data than the Gossip-based algorithms. 
However, as shown in Eq. \eqref{eq:update_z}, 
the node $i$ must receive dual variables $\mathbf{y}_{j|i}$ from its neighbor $j$ in the ECL.
Therefore, in the ECL as well as in the Gossip-based algorithm,
large communication costs occur during training.

In addition to the ECL, other methods using primal-dual formulations have been proposed \citep{li2019decentralized},
and recently, the compression methods for these primal-dual algorithms has been studied \cite{kovalev2021linearly,liu2021linear}.
However, the compression methods for the ECL have not been studied.
In this work, we propose the compression method for the ECL, called the C-ECL, 
that can train a model with fewer parameter exchanges and is robust to heterogeneous data.

\section{Proposed Method}
\label{sec:proposed}

\subsection{Compression Operator}
Before proposing the C-ECL, 
we first introduce the compression operator used in this work.
\begin{assumption}[Compression Operator]
\label{assumption:compression}
For some $\tau \in (0, 1]$,
we assume that the compression operator $\textbf{comp}: \mathbb{R}^d \rightarrow \mathbb{R}^d$ satisfies the following conditions:
\begin{alignat}{3}
    \label{eq:bound_of_compression}
    \mathbb{E}_\omega \left\| \mathrm{comp}(\mathbf{x} ; \omega) - \mathbf{x} \right\|^2  &\leq (1-\tau) \left\| \mathbf{x} \right\|^2 \quad &(\forall \mathbf{x} \in \mathbb{R}^d), \\
    \label{eq:linearlity_of_compression}
    \mathrm{comp}(\mathbf{x} + \mathbf{y} ; \omega) &= \mathrm{comp} (\mathbf{x} ; \omega) + \mathrm{comp} (\mathbf{y} ; \omega) \quad &(\forall \omega, \forall \mathbf{x}, \mathbf{y} \in \mathbb{R}^d), \\
    \label{eq:linearlity_of_compression_2}
    \mathrm{comp}(-\mathbf{x} ; \omega) &= - \mathrm{comp}(\mathbf{x} ; \omega) \quad &(\forall \omega, \forall \mathbf{x} \in \mathbb{R}^d),
\end{alignat}
where $\omega$ represents the parameter of the compression operator.
In the following, we abbreviate $\omega$ and write $\textbf{comp}(\mathbf{x})$ as the operator containing the randomness.
\end{assumption}
The assumption of Eq. \eqref{eq:bound_of_compression} is commonly used for the compression methods for the Gossip-based algorithms \citep{lu2020monniqua,vogels2020powergossip,koloskova2019decentralized}.
In addition, we assume that the compression operator satisfies Eqs. (\ref{eq:linearlity_of_compression}-\ref{eq:linearlity_of_compression_2}),
and the low-rank approximation \cite{vogels2020powergossip} and the following sparsification used in the compression methods for the Gossip-based algorithm satisfy Eqs. (\ref{eq:linearlity_of_compression}-\ref{eq:linearlity_of_compression_2}).

\begin{example}
For some $k \in (0, 100]$,
we define the operator $\textbf{rand}_{k\%}: \mathbb{R}^d \rightarrow \mathbb{R}^d$ as follows:
\begin{align}
\label{eq:rand}
    \mathrm{rand}_{k \%}(\mathbf{x}) \coloneqq \mathbf{s} \circ \mathbf{x},
\end{align}
where $\circ$ is the Hadamard product and $\mathbf{s} \in \{0, 1\}^d$ is a uniformly sampled sparse vector whose element is one with probability $k \%$.
Here, the parameter $\omega$ of the compression operator corresponds to the randomly sampled vector $\mathbf{s}$.
Then, $\textbf{rand}_{k\%}$ satisfies Assumption \ref{assumption:compression} \citep{stich2018sparsified}.
\end{example}

\subsection{Communication Compressed Edge-Consensus Learning}
In this section, we propose the \textit{the Communication Compressed ECL (C-ECL)},
the method for compressing the dual variables exchanged in the ECL using the compression operator introduced in the previous section.

In the ECL, to update $\mathbf{z}_{i|j}$ in Eq. \eqref{eq:update_z}, the node $i$ needs to receive $\mathbf{y}_{j|i}$ from the node $j$.
Because the number of elements in $\mathbf{y}_{j|i}$ is the same as that of the model parameter $\mathbf{w}_j$,
this exchange incurs significant communication costs.
A straightforward approach to reduce this communication cost is compressing $\mathbf{y}_{j|i}$ in Eq. \eqref{eq:update_z} as follows:
\begin{align}
\label{eq:naive_compression}
    \mathbf{z}^{(r+1)}_{i|j} &= (1 - \theta) \mathbf{z}^{(r)}_{i|j} + \theta \; \text{comp} ( \mathbf{y}^{(r+1)}_{j|i} ).
\end{align}
However, we found experimentally that compressing $\mathbf{y}_{j|i}$ does not work.
In the compression methods for the Gossip-based algorithm, \citet{lu2020monniqua} showed that the model parameters are not robust to the compression.
This is because the optimal solution of the model parameter is generally not zero, 
and thus the error caused by the compression does not approach zero even if the model parameters are near the optimal solution.
Therefore, the successful compression methods for the Gossip-based algorithms  
compress the gradient $\nabla f_i(\mathbf{w}_i)$ or the model difference $(\mathbf{w}_j - \mathbf{w}_i)$,
which approach zero when the model parameters are near the optimal solution \cite{koloskova2019decentralized,lu2020monniqua,vogels2020powergossip}.

Inspired by these compression methods for the Gossip-based algorithms,
we reformulate Eq. \eqref{eq:update_z} into Eq. \eqref{eq:reformulated_update_of_z} so that we can compress the parameters
which approach zero when the model parameters are near the optimal solution.
\begin{align}
\label{eq:reformulated_update_of_z}
    \mathbf{z}^{(r+1)}_{i|j} &= \mathbf{z}^{(r)}_{i|j} + \theta (\mathbf{y}^{(r+1)}_{j|i} - \mathbf{z}^{(r)}_{i|j}).
\end{align}
In the Douglas-Rachford splitting, $\mathbf{z}_{i|j}$ approaches the fixed point (i.e., $\mathbf{z}_{i|j}^{(r)}=\mathbf{z}_{i|j}^{(r+1)}$)
when the model parameters approach the optimal solution.
(See Sec. \ref{sec:preliminary} for details of the Douglas-Rachford splitting and the definition of the fixed point).
Then, from Eq. \eqref{eq:reformulated_update_of_z}, 
when the model parameters approach to the optimal solution,
$(\mathbf{y}_{j|i} - \mathbf{z}_{i|j})$ in Eq. \eqref{eq:reformulated_update_of_z} approaches zero.
Then, instead of compressing $\mathbf{y}_{j|i}$ as in Eq. \eqref{eq:naive_compression},
we propose compressing $(\mathbf{y}_{j|i} - \mathbf{z}_{i|j})$ as follows:
\begin{align}
\label{eq:compression_in_cecl}
    \mathbf{z}^{(r+1)}_{i|j} 
    &= \mathbf{z}^{(r)}_{i|j} + \theta \; \text{comp} (\mathbf{y}^{(r+1)}_{j|i} - \mathbf{z}^{(r)}_{i|j}) \\
    &= \mathbf{z}^{(r)}_{i|j} + \theta \; (\text{comp} (\mathbf{y}^{(r+1)}_{j|i}) - \text{comp}(\mathbf{z}^{(r)}_{i|j})), \nonumber 
\end{align}
where we use Assumption \ref{assumption:compression} in the last equation.
When $\textbf{rand}_{k\%}$ is used as the compression operator,
$\mathbf{y}_{j|i}$ and $\mathbf{z}_{i|j}$ must be compressed using the same sparse vector $\mathbf{s}$ in Eq. \eqref{eq:rand}
to use Assumption \ref{assumption:compression}.
In Alg. \ref{alg:C-ECL}, we provide the pseudo-code of the C-ECL.
For simplicity, the node $i$ and the node $j$ exchange $\omega_{i|j}$ and $\omega_{j|i}$ at Lines 5-6 in Alg. \ref{alg:C-ECL}.
However, by sharing the same seed value to generate $\omega_{i|j}$ and $\omega_{j|i}$ before starting the training,
the node $i$ and the node $j$ can obtain the same values $\omega_{i|j}$ and $\omega_{j|i}$ without any exchanges.
Moreover, when $\textbf{rand}_{k\%}$ is used as the compression operator,
the node $i$ can obtain $\omega_{i|j}$ from the received value $\textbf{comp}(\mathbf{y}_{j|i} ; \omega_{i|j})$
because $\textbf{comp}(\mathbf{y}_{j|i} ; \omega_{i|j})$ is stored in a sparse matrix format (e.g., COO format).
Thus, these exchanges of $\omega_{i|j}$ and $\omega_{j|i}$ can be omitted in practice.
Therefore, in the C-ECL, each node only needs to exchange the compressed value of $\mathbf{y}_{j|i}$,
and the C-ECL can train a model with fewer parameter exchanges than the ECL.
\begin{algorithm}[t]
   \caption{Update procedure at the node $i$ of the C-ECL.}
   \label{alg:C-ECL}
\begin{algorithmic}[1]
   \FOR{$r=0$ {\bfseries to} $R$}
   \STATE $\mathbf{w}^{(r+1)}_i \leftarrow \text{argmin}_{\mathbf{w}_i} \{ f_i(\mathbf{w}_i)
    + \frac{\alpha}{2} \sum_{j\in\mathcal{N}_i} {\left\| \mathbf{A}_{i|j} \mathbf{w}_i - \frac{1}{\alpha} \mathbf{z}^{(r)}_{i|j} \right\|}^2 \}$.
   \FOR{$j \in \mathcal{N}_i$ }
   \STATE $\mathbf{y}^{(r+1)}_{i|j} \leftarrow \mathbf{z}^{(r)}_{i|j} - 2 \alpha \mathbf{A}_{i|j} \mathbf{w}^{(r+1)}_i$.
   \STATE $\textbf{Receive}_{i\leftarrow j}(\omega_{i|j}^{(r+1)})$.  \tcp*[f]{This exchange can be omitted.}
   \STATE $\textbf{Transmit}_{i\rightarrow j}(\omega_{j|i}^{(r+1)})$.  \tcp*[f]{This exchange can be omitted.}
   \STATE $\textbf{Receive}_{i\leftarrow j}(\text{comp}( \mathbf{y}^{(r+1)}_{j|i} ; \omega^{(r+1)}_{i|j}))$.
   \STATE $\textbf{Transmit}_{i\rightarrow j}( \text{comp}(\mathbf{y}^{(r+1)}_{i|j} ; \omega^{(r+1)}_{j|i}))$.
   \STATE $\mathbf{z}^{(r+1)}_{i|j} \leftarrow \mathbf{z}^{(r)}_{i|j} + \theta \; \text{comp}(\mathbf{y}^{(r+1)}_{j|i} - \mathbf{z}^{(r)}_{i|j} ; \omega^{(r+1)}_{i|j}) $.
   \ENDFOR
   \ENDFOR
\end{algorithmic}
\end{algorithm}

\section{Convergence Analysis}
In this section, we analyze how compression in the C-ECL affects the convergence rate of the ECL.
Our convergence analysis is based on the analysis of the Douglas-Rachford splitting \citep{giselsson2017linear},
and the proofs of which are presented in Sec. \ref{sec:convergence_analysis}.

\subsection{Assumptions}
In this section, we introduce additional notations and assumptions used in our convergence analysis.
We define $N\coloneqq |\mathcal{V}|$, $N_{\text{min}} \coloneqq \text{min}_i \{ |\mathcal{N}_i| \}$, 
$N_{\text{max}} \coloneqq \text{max}_i \{ |\mathcal{N}_i| \}$
where $\mathcal{N}_i$ is the set of neighbors of the node $i$.
Let $\mathcal{N}_i(j)$ be the $j$-th smallest index of the node in $\mathcal{N}_i$,
we define $\mathbf{w} \in \mathbb{R}^{dN}$,
$\mathbf{z}_i \in \mathbb{R}^{d|\mathcal{N}_i|}$,
and $\mathbf{z} \in \mathbb{R}^{2d|\mathcal{E}|}$ as follows:
\begin{align}
\label{eq:additional_definition}
    \mathbf{w} \coloneqq (\mathbf{w}_1^\top, \ldots, \mathbf{w}_N^\top)^\top, \;\;
    \mathbf{z}_i \coloneqq (\mathbf{z}_{i|\mathcal{N}_i(1)}^\top, \ldots, \mathbf{z}_{i|\mathcal{N}_i(|\mathcal{N}_i|)}^\top)^\top, \;\;
    \mathbf{z} \coloneqq (\mathbf{z}_1^\top, \ldots, \mathbf{z}_N^\top)^\top.
\end{align}
For simplicity, we drop the superscripts of the number of round $r$.
Let $\{ \mathbf{w}_i^{\ast} \}_i$ be the optimal solution of Eq. \eqref{eq:ecl_primal},
we define $\mathbf{w}^\star \in \mathbb{R}^{dN}$ in the same manner as the definition of $\mathbf{w}$ in Eq. \eqref{eq:additional_definition}.
We define the loss function as $f(\mathbf{w}) \coloneqq \sum_{i\in\mathcal{V}} f_i(\mathbf{w}_i)$.
Next, we introduce the assumptions used in the convergence analysis.
\begin{assumption}
\label{assumption:convex}
We assume that $f$ is proper, closed and convex.
\end{assumption}
\begin{assumption}
\label{assumption:smoothness_and_strong_convex}
We assume that $f$ is $L$-smooth and $\mu$-strongly convex with $L>0$ and $\mu>0$.
\end{assumption}
\begin{assumption}
\label{assumption:no_isolated_nodes}
We assume that the graph $G$ has no isolated nodes (i.e., $N_{\text{min}}>0$).
\end{assumption}
Assumption \ref{assumption:convex}, \ref{assumption:smoothness_and_strong_convex} are
the standard assumptions used for the convergence analysis of the operator splitting methods \citep{giselsson2017linear,ryu2015primer}.
Assumption \ref{assumption:smoothness_and_strong_convex} is weaker than the assumptions of the smoothness and the strongly convexity of $f_i$ for all $i \in \mathcal{V}$, which are commonly used in decentralized learning.
Assumption \ref{assumption:no_isolated_nodes} holds in general because decentralized learning assumes that the graph $G$ is connected.
In addition, we define $\delta \in \mathbb{R}$ as follows:
\begin{align*}
    \delta \coloneqq \text{max} \left( 
    \frac{\alpha  N_{\text{max}} - \mu}{\alpha  N_{\text{max}} + \mu},
    \frac{L - \alpha  N_{\text{min}}}{L + \alpha  N_{\text{min}}}
    \right).
\end{align*}
Suppose that Assumptions \ref{assumption:convex}, \ref{assumption:smoothness_and_strong_convex}, and \ref{assumption:no_isolated_nodes} hold
and $\alpha \in (0, \infty)$ holds,
then $\delta \in [0, 1)$ holds because $L\geq\mu>0$ and $N_{\text{max}}\geq N_{\text{min}}>0$.

\subsection{Convergence Rates}
\begin{theorem}
\label{theorem:main_convergence_analysis}
Let $\bar{\mathbf{z}} \in \mathbb{R}^{2d|\mathcal{E}|}$ be the fixed point of the Douglas-Rachford splitting.\footnote{A more detailed definition is shown in Sec. \ref{sec:preliminary} and \ref{sec:convergence_analysis}.}
Suppose that Assumptions \ref{assumption:compression}, \ref{assumption:convex}, \ref{assumption:smoothness_and_strong_convex}, and \ref{assumption:no_isolated_nodes} hold.
If $\tau \geq 1 - ( \frac{1 - \delta}{1 + \delta} )^2$ and $\theta$ satisfies
\begin{align}
\label{eq:condition_of_theta}
    \theta \in \left( \frac{2\delta \sqrt{1 - \tau}}{(1 - \delta)(1 - \sqrt{1 - \tau})}, \frac{2}{(1+\delta)(1 + \sqrt{1 - \tau})} \right),
\end{align}
then $\mathbf{w}^{(r+1)}$ generated by Alg. \ref{alg:C-ECL} linearly converges to the optimal solution $\mathbf{w}^\star$ of Eq. \eqref{eq:ecl_primal} as follows:
\begin{align}
\label{eq:convergence_rate_of_cecl}
    \mathbb{E} \|\mathbf{w}^{(r+1)} - \mathbf{w}^\star \| 
    \leq \frac{\sqrt{N_{\text{max}}}}{\mu + \alpha N_{\text{min}}} \left\{ |1 - \theta| + \theta \delta + \sqrt{1 - \tau} (\theta + |1 - \theta| \delta + \delta ) \right\}^r \| \mathbf{z}^{(0)} - \bar{\mathbf{z}} \|.
\end{align}
\end{theorem}
\begin{corollary}
\label{corollary:convergenece_rate_of_ecl}
Let $\bar{\mathbf{z}} \in \mathbb{R}^{2d|\mathcal{E}|}$ be the fixed point of the Douglas-Rachford splitting.
Under Assumptions \ref{assumption:compression}, \ref{assumption:convex}, \ref{assumption:smoothness_and_strong_convex}, and \ref{assumption:no_isolated_nodes},
when $\tau = 1$ and $\theta \in ( 0, \frac{2}{1+\delta} )$,
$\mathbf{w}^{(r+1)}$ generated by Alg. \ref{alg:C-ECL} linearly converges to the optimal solution $\mathbf{w}^\star$ of Eq. \eqref{eq:ecl_primal} as follows:
\begin{align}
\label{eq:convergence_rate_of_ecl}
    \mathbb{E} \|\mathbf{w}^{(r+1)} - \mathbf{w}^\star \| \leq \frac{\sqrt{N_{\text{max}}}}{\mu + \alpha N_{\text{min}}} ( |1 - \theta| + \theta \delta )^r \| \mathbf{z}^{(0)} - \bar{\mathbf{z}} \|.
\end{align}
\end{corollary}

Because $\tau=1$ implies that $\textbf{comp}(\mathbf{x})=\mathbf{x}$ in the C-ECL,
Corollary \ref{corollary:convergenece_rate_of_ecl} shows the convergence rate of the ECL under Assumptions \ref{assumption:compression}, \ref{assumption:convex}, \ref{assumption:smoothness_and_strong_convex}, and \ref{assumption:no_isolated_nodes},
which is almost the same rate as that shown in the previous work \citep{giselsson2017linear}.
Comparing the domain of $\theta$ for the ECL and the C-ECL to converge, 
as $\tau$ decreases, the domain in Eq. \eqref{eq:condition_of_theta} becomes smaller.
Subsequently, in order for the domain in Eq. \eqref{eq:condition_of_theta} to be non-empty,
$\tau$ must be greater than or equal to $(1 - (1-\delta)^2/(1+\delta)^2)$.
Next, comparing the convergence rate of the ECL and the C-ECL,
the compression in the C-ECL slows down the convergence rate of the ECL by the term $(\sqrt{1 - \tau} (\theta + |1 - \theta| \delta + \delta ))$.
Moreover, similar to the convergence analysis of the Douglas-Rachford splitting \citep{giselsson2017linear},
Theorem \ref{theorem:main_convergence_analysis} and Corollary \ref{corollary:convergenece_rate_of_ecl}
imply that the optimal parameter of $\theta$ can be determined as follows:

\begin{corollary}
\label{corollary:optimal_convergence_rate}
Suppose that \ref{assumption:compression}, \ref{assumption:convex}, \ref{assumption:smoothness_and_strong_convex}, and \ref{assumption:no_isolated_nodes} hold,
and $\tau \geq 1 - ( \frac{1 - \delta}{1 + \delta} )^2$,
the optimal convergence rate of Eq. \eqref{eq:convergence_rate_of_cecl} in the C-ECL is achieved when $\theta=1$.
\end{corollary}
\begin{corollary}
\label{corollary:optimal_convergence_rate_of_ecl}
Suppose that Assumption \ref{assumption:compression}, \ref{assumption:convex}, \ref{assumption:smoothness_and_strong_convex}, and \ref{assumption:no_isolated_nodes} hold,
and $\tau=1$,
the optimal convergence rate of Eq. \eqref{eq:convergence_rate_of_ecl} is achieved when $\theta=1$.
\end{corollary}

\section{Experiments}
In this section, we demonstrate that the C-ECL can achieve almost the same performance with fewer parameter exchanges than the ECL.
Furthermore, we show that the C-ECL is more robust to heterogeneous data than the Gossip-based algorithm.

\subsection{Experimental Setting}
\textbf{Dataset and Model:}
We evaluate the C-ECL using FashionMNIST \cite{xiao2017/online} and CIFAR10 \citep{Krizhevsky09learningmultiple},
which are datasets of 10-class image-classification tasks.
As the models for both datasets, we use 5-layer convolutional neural networks \citep{lecun1998gradientbased} with group normalization \cite{wu2018group}.
Following the previous work \cite{niwa2020edge}, 
we distribute data to the nodes in two settings: the homogeneous and heterogeneous settings.
In the homogeneous setting, the data are distributed such that each node has the data of all $10$ classes
and has approximately the same number of data of each class.
In the heterogeneous setting, the data are distributed such that each node has the data of randomly selected $8$ classes.
In both settings, the data are distributed to the nodes such that each node has the same number of data.

\textbf{Network:}
In Sec. \ref{sec:experimental_results}, we evaluate all comparison methods on a network of a ring consisting of eight nodes.
In addition, in Sec. \ref{sec:network_topology},
we evaluate all comparison methods in four settings of the network topology: chain, ring, multiplex ring, and fully connected graph,
where each setting consists of eight nodes.
In Sec. \ref{sec:experimental_setting}, we show a visualization of the network topology.
Each node exchanges parameters with its neighbors per five local updates.
We implement with PyTorch using gloo\footnote{\url{https://pytorch.org/docs/stable/distributed.html}} as the backend,
and run all comparison methods on eight GPUs (NVIDIA RTX 3090).

\textbf{Comparison Methods:}
(1) D-PSGD \citep{lian2017can}: The uncompressed Gossip-based algorithm.
(2) PowerGossip \citep{vogels2020powergossip}: The Gossip-based algorithm that compresses the exchanging parameters by using a low-rank approximation. We use the PowerGossip as the compression method for the Gossip-based algorithm
because the PowerGossip has been shown to achieve almost the same performance as other existing compression methods
without additional hyperparameter tuning.
(3) ECL \citep{niwa2020edge,niwa2021asynchronous}: The primal-dual algorithm described in Sec. \ref{sec:primal_dual}.
Because \citet{niwa2020edge} showed that the ECL converges faster when $\theta=1$ than when $\theta=0.5$,
we set $\theta=1$.
(4) C-ECL: Our proposed method described in Sec. \ref{sec:proposed}.
We use $\textbf{rand}_{k\%}$ as the compression operator.
Following Corollary \ref{corollary:optimal_convergence_rate}, we set $\theta=1$.
We initialize $\mathbf{z}_{i|j}$ and $\mathbf{y}_{i|j}$ to zeros.
However, we found that when we compress the update values of $\mathbf{z}_{i|j}$ by using $\textbf{rand}_{k\%}$,
the convergence becomes slower because $\mathbf{z}_{i|j}$ remains to be sparse in the early training stage.
Then, we set $k\%$ of $\textbf{rand}_{k\%}$ to $100\%$ only during the first epoch.

In addition, for the reference, we show the results of the Stochastic Gradient Descent (SGD),
in which the model is trained on a single node containing all training data. 
In our experiments, we set the learning rate, number of epochs, and batch size to the same values for all comparison methods.
In Sec. \ref{sec:experimental_setting}, we show the detailed hyperparameters used for all comparison methods.

\subsection{Experimental Results}
\label{sec:experimental_results}
In this section, we evaluate the accuracy and the communication costs when setting the network topology to be a ring.

\begin{table}[t]
\center
\caption{Test accuracy and communication costs on the homogeneous setting.
For the C-ECL, the number in the bracket is $k$ of $\textbf{rand}_{k\%}$.
For the PowerGossip, the number in the bracket is the number of the power iteration steps.
As the communication costs, the average amount of parameter sent per epoch is shown.}
\label{table:iid}
\vskip -0.1 in
\begin{tabular}{lcrrccrrr}
               & \multicolumn{3}{c}{FashionMNIST} && \multicolumn{3}{c}{CIFAR10} \\
               & Accuracy       & \multicolumn{2}{c}{Send/Epoch}      && Accuracy    & \multicolumn{2}{c}{Send/Epoch} \\
\cmidrule[\heavyrulewidth]{1-4} 
\cmidrule[\heavyrulewidth]{6-8}
SGD             &  $88.7$  & \multicolumn{2}{c}{-} &&  $75.7$  &     \multicolumn{2}{c}{-} \\
\cmidrule{1-4} 
\cmidrule{6-8}
D-PSGD          &  $84.1$  &  $5336$ KB  & ($\times 1.0$)  && $72.8$  & $6255$ KB & ($\times 1.0$) \\
ECL             &  $84.4$  &  $5336$ KB  & ($\times 1.0$)  && $72.6$  & $6255$ KB & ($\times 1.0$) \\
\cmidrule{1-4} 
\cmidrule{6-8}
PowerGossip (1)  &  $84.0$  &   $138$ KB  & ($\times 38.7$)  && $72.0$  &  $135$ KB & ($\times 46.3$) \\
PowerGossip (10) &  $84.3$  &   $1079$ KB  & ($\times 5.0$)  && $72.3$  &  $1102$ KB & ($\times 5.7$) \\
PowerGossip (20) &  $84.2$  &   $2124$ KB  & ($\times 2.5$)  && $72.2$  &  $2175$ KB & ($\times 2.9$) \\
\cmidrule{1-4} 
\cmidrule{6-8}
C-ECL ($1\%$)   &  $84.0$   &   $115$ KB & ($\times 48.1$) && $71.5$  &  $132$ KB & ($\times 47.4$) \\
C-ECL ($10\%$)  &  $84.0$   &  $1075$ KB & ($\times 5.1$)  && $71.4$  & $1257$ KB & ($\times 5.0$) \\
C-ECL ($20\%$)  &  $84.0$   &  $2142$ KB & ($\times 2.5$)  && $71.1$  & $2507$ KB & ($\times 2.5$) \\
\cmidrule[\heavyrulewidth]{1-4} 
\cmidrule[\heavyrulewidth]{6-8}
\end{tabular}
\end{table}

\begin{table}[t]
\center
\caption{Test accuracy and communication costs on the heterogeneous setting.
For the C-ECL, the number in the bracket is $k$ of $\textbf{rand}_{k\%}$.
For the PowerGossip, the number in the bracket is the number of the power iteration steps.
As the communication costs, the average amount of parameter sent per epoch is shown.}
\label{table:noniid}
\vskip -0.1 in
\begin{tabular}{lcrrccrrr}
               & \multicolumn{3}{c}{FashionMNIST} && \multicolumn{3}{c}{CIFAR10} \\
               & Accuracy       & \multicolumn{2}{c}{Send/Epoch}      && Accuracy    & \multicolumn{2}{c}{Send/Epoch} \\
\cmidrule[\heavyrulewidth]{1-4} 
\cmidrule[\heavyrulewidth]{6-8}
SGD             &  $88.7$ & \multicolumn{2}{c}{-} && $75.7$ &     \multicolumn{2}{c}{-} \\
\cmidrule{1-4} 
\cmidrule{6-8}
D-PSGD          &  $79.4$  &  $5336$ KB  & ($\times 1.0$)  &&  $70.8$   &  $6155$ KB  & ($\times 1.0$) \\
ECL             &  $84.5$  &  $5336$ KB  & ($\times 1.0$)  &&  $72.7$   &  $6155$ KB  & ($\times 1.0$) \\
\cmidrule{1-4} 
\cmidrule{6-8}
PowerGossip (1)  &  $77.5$  &   $138$ KB  & ($\times 38.7$)  && $64.3$  &  $133$ KB & ($\times 46.3$) \\
PowerGossip (10) &  $77.7$  &   $1079$ KB  & ($\times 5.0$)  && $67.2$  &  $1084$ KB & ($\times 5.7$) \\
PowerGossip (20) &  $77.4$  &   $2124$ KB  & ($\times 2.5$)  && $67.9$  &  $2141$ KB & ($\times 2.9$) \\
\cmidrule{1-4} 
\cmidrule{6-8}
C-ECL ($1\%$)   &  $77.7$  &   $115$ KB  & ($\times 48.1$) &&  $60.2$   &   $129$ KB  & ($\times 47.7$)\\
C-ECL ($10\%$)  &  $83.4$  &  $1075$ KB  &  ($\times 5.1$) &&  $61.8$   &  $1237$ KB  &  ($\times 5.0$)\\
C-ECL ($20\%$)  &  $83.6$  &  $2142$ KB  &  ($\times 2.5$) &&  $72.3$   &  $2467$ KB  &  ($\times 2.5$)\\
\cmidrule[\heavyrulewidth]{1-4} 
\cmidrule[\heavyrulewidth]{6-8}
\end{tabular}
\vskip -0.2 in
\end{table}
\textbf{Homogeneous Setting:}
First, we discuss the results on the homogeneous setting.
Table \ref{table:iid} shows the accuracy and the communication costs on the homogeneous setting.
The results show that the D-PSGD and the ECL achieve almost the same accuracy on both datasets.
Then, the C-ECL and the PowerGossip are comparable and achieve almost the same accuracy as the ECL and the D-PSGD 
even when we set $k\%$ of $\textbf{rand}_{k\%}$ to $1\%$ and the number of the power iteration steps to $1$ respectively.
Therefore, the C-ECL can achieve the comparable accuracy with approximately $50$-times fewer parameter exchanges than the ECL and the D-PSGD
on the homogeneous setting.

\textbf{Heterogeneous Setting}:
Next, we discuss the results on the heterogeneous setting.
Table \ref{table:noniid} shows the accuracy and the communication costs on the heterogeneous setting.
In the D-PSGD, the accuracy on the heterogeneous setting decreases by approximately $3\%$ compared to that on the homogeneous setting.
In the PowerGossip, even if the number of power iteration steps is increased, 
the accuracy does not approach that of the D-PSGD and the ECL.
On the other hand, the accuracy of the ECL is almost the same on both the homogeneous and heterogeneous settings,
and the results indicate that the ECL is more robust to heterogeneous data than the D-PSGD.
In the C-ECL, when we set $k\%$ of $\textbf{rand}_{k\%}$ to $1\%$, 
the accuracy on the heterogeneous setting decreases by approximately $10\%$ compared to that of the ECL.
However, when we increase $k\%$ of $\textbf{rand}_{k\%}$, 
the C-ECL is competitive with the ECL and outperforms the D-PSGD and the PowerGossip.
Specifically, on FashionMNIST, when we set $k\%$ to $10\%$, the C-ECL is competitive to the ECL
and outperforms the D-PSGD and the PowerGossip.
On CIFAR10, when we set $k\%$ to $20\%$, the C-ECL is competitive with the ECL
and outperforms the D-PSGD and the PowerGossip.

In summary, on the homogeneous setting, 
the C-ECL and the PowerGossip can achieve almost the same accuracy as the ECL and the D-PSGD with approximately $50$-times fewer parameter exchanges.
On the heterogeneous setting, 
the C-ECL can achieve almost the same accuracy as the ECL with approximately $4$-times fewer parameter exchanges
and can outperform the PowerGossip.
Furthermore, the results show that the C-ECL can outperform the D-PSGD, the uncompressed Gossip-based algorithm,
in terms of both the accuracy and the communication costs.

\subsection{Network Topology}
\label{sec:network_topology}
In this section, we show the accuracy and the communication costs when the network topology is varied.
Table \ref{table:cost} and Fig. \ref{fig:fashion} show the communication costs and the accuracy on FashionMNIST
when the network topology is varied as a chain, ring, multiplex ring, or fully connected graph.

On the homogeneous setting, Fig. \ref{fig:fashion} shows that the accuracy of all comparison methods are almost the same
and reach that of the SGD on all network topologies.
On the heterogeneous setting, Fig. \ref{fig:fashion} shows that the accuracy of the D-PSGD and the PowerGossip are decreased compared to that on the homogeneous setting. 
On the other hand, the accuracy of the ECL is almost the same as on the homogeneous setting on all network topologies.
Then, on all network topologies, the C-ECL achieves almost the same accuracy as the ECL with fewer parameters exchanges
and consistently outperforms the PowerGossip.
Moreover, the results show that on the heterogeneous setting,
the C-ECL can outperform the D-PSGD, the uncompressed Gossip-based algorithm, 
on all network topologies in terms of both the accuracy and the communication costs.
\begin{figure}[t]
\center
\subfigure[Homogeneous Setting]{
    \includegraphics[width=0.95\hsize]{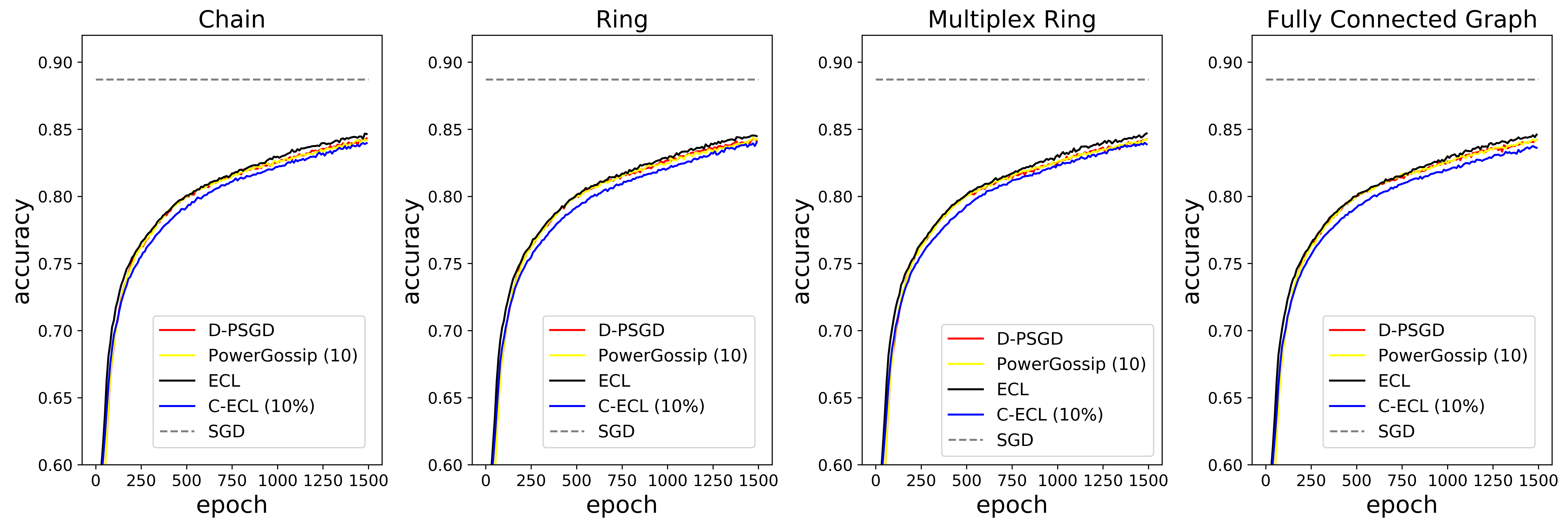}
}
\vskip -0.05 in
\subfigure[Heterogeneous Setting]{
    \includegraphics[width=0.95\hsize]{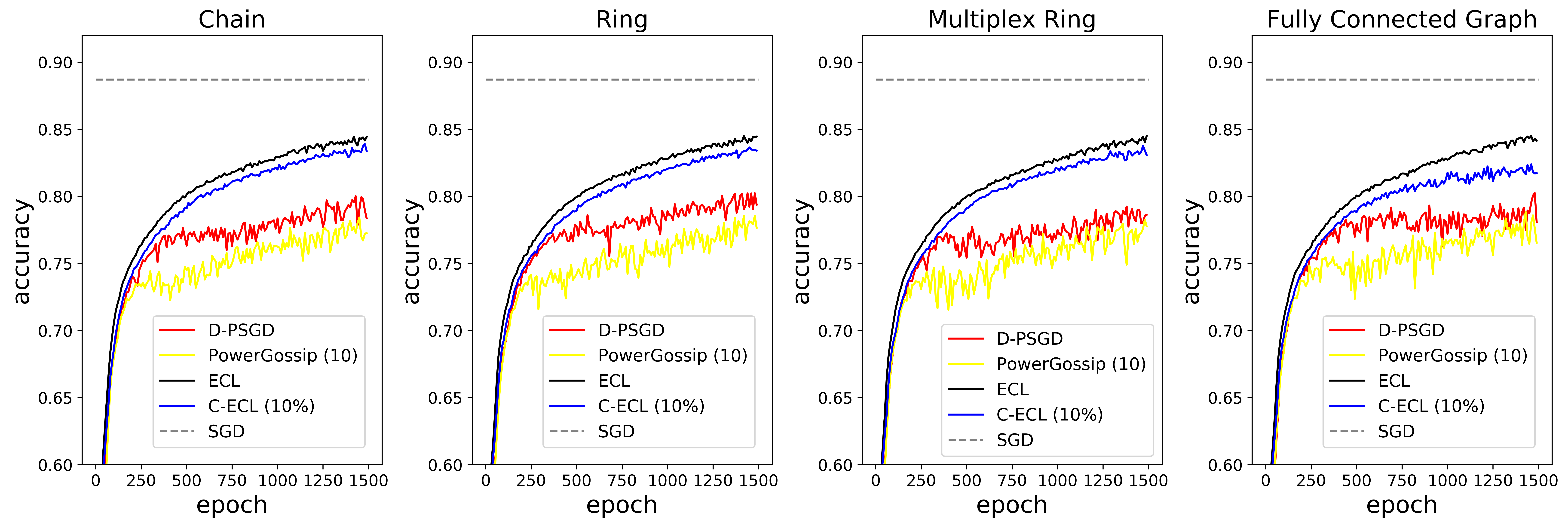}
}
\vskip -0.1 in
\caption{Test accuracy on FashionMNIST when varying the network topology.
We evaluate the average test accuracy of each node per 10 epochs.}
\label{fig:fashion}
\end{figure}
\begin{table}[t]
\center
\caption{Communication costs on FashionMNIST when varying the network topology. 
As the communication costs, the average amount of parameters sent per epoch on homogeneous and heterogeneous settings is shown.}
\label{table:cost}
\begin{tabular}{lrrrr}
                 & \multicolumn{1}{c}{Chain} & \multicolumn{1}{c}{Ring} & \multicolumn{1}{c}{Multiplex Ring} & \multicolumn{1}{c}{Fully Connected Graph} \\
\toprule
D-PSGD           & $4670$ KB & $5336$ KB & $10673$ KB & $18677$ KB  \\
ECL              & $4670$ KB & $5336$ KB & $10673$ KB & $18677$ KB  \\
PowerGossip (10) &  $944$ KB & $1078$ KB &  $2158$ KB &  $3776$ KB  \\
C-ECL ($10\%$)   &  $941$ KB & $1075$ KB &  $2151$ KB &  $3764$ KB  \\
\bottomrule
\end{tabular}
\vskip -0.15 in
\end{table}

\section{Conclusion}
In this work, we propose the Communication Compressed ECL (C-ECL), 
the novel framework for the compression methods of the ECL.
Specifically, we reformulate the update formula of the ECL and propose compressing the update values of the dual variables.
Theoretically, we analyze the convergence rate of the C-ECL
and show that the C-ECL converges linearly to the optimal solution as well as the ECL.
Experimentally, we demonstrate that, even if the data distribution of each node is statistically heterogeneous,
the C-ECL can achieve almost the same accuracy as the ECL with fewer parameter exchanges.
Moreover, we show that,
when the data distribution of each node is statistically heterogeneous, 
the C-ECL outperforms the uncompressed Gossip-based algorithm in terms of both the accuracy and the communication costs.

\section*{Acknowledgments}
M.Y. was supported by MEXT KAKENHI Grant Number 20H04243.

\bibliography{ref}

\begin{thebibliography}{}

\bibitem[Bauschke and Combettes, 2017]{bauschke2011convex}
Bauschke, H.~H. and Combettes, P.~L. (2017).
\newblock {\em Convex analysis and monotone operator theory in hilbert spaces}.
\newblock Springer, 2nd edition.

\bibitem[Boyd et~al., 2006]{boyd2006randomized}
Boyd, S., Ghosh, A., Prabhakar, B., and Shah, D. (2006).
\newblock Randomized gossip algorithms.
\newblock In {\em IEEE Transactions on Information Theory}.

\bibitem[Boyd et~al., 2011]{boyd2011distributed}
Boyd, S., Parikh, N., Chu, E., Peleato, B., and Eckstein, J. (2011).
\newblock Distributed optimization and statistical learning via the alternating
  direction method of multipliers.
\newblock {\em Foundations and Trends in Machine Learning}.

\bibitem[Chen et~al., 2020]{chen2020simple}
Chen, T., Kornblith, S., Norouzi, M., and Hinton, G. (2020).
\newblock A simple framework for contrastive learning of visual
  representations.
\newblock In {\em International Conference on Machine Learning}.

\bibitem[Defazio et~al., 2014]{defazio2014saga}
Defazio, A., Bach, F., and Lacoste-Julien, S. (2014).
\newblock Saga: A fast incremental gradient method with support for
  non-strongly convex composite objectives.
\newblock In {\em Advances in Neural Information Processing Systems}.

\bibitem[Devlin et~al., 2019]{devlin2019bert}
Devlin, J., Chang, M.-W., Lee, K., and Toutanova, K. (2019).
\newblock {BERT}: Pre-training of deep bidirectional transformers for language
  understanding.
\newblock In {\em Conference of the North {A}merican Chapter of the Association
  for Computational Linguistics}.

\bibitem[Dosovitskiy et~al., 2021]{dosovitskiy2021an}
Dosovitskiy, A., Beyer, L., Kolesnikov, A., Weissenborn, D., Zhai, X.,
  Unterthiner, T., Dehghani, M., Minderer, M., Heigold, G., Gelly, S.,
  Uszkoreit, J., and Houlsby, N. (2021).
\newblock An image is worth 16x16 words: Transformers for image recognition at
  scale.
\newblock In {\em International Conference on Learning Representations}.

\bibitem[Douglas and Rachford, 1956]{douglas1955numerical}
Douglas, J. and Rachford, H.~H. (1956).
\newblock On the numerical solution of heat conduction problems in two and
  three space variables.
\newblock {\em Transactions of the American mathematical Society}.

\bibitem[Giselsson and Boyd, 2017]{giselsson2017linear}
Giselsson, P. and Boyd, S.~P. (2017).
\newblock Linear convergence and metric selection for douglas-rachford
  splitting and {ADMM}.
\newblock {\em {IEEE} Transactions on Automatic Control}.

\bibitem[Johnson and Zhang, 2013]{johnson2013accelerating}
Johnson, R. and Zhang, T. (2013).
\newblock Accelerating stochastic gradient descent using predictive variance
  reduction.
\newblock In {\em Advances in Neural Information Processing Systems}.

\bibitem[Karimireddy et~al., 2020]{karimireddy2020scaffold}
Karimireddy, S.~P., Kale, S., Mohri, M., Reddi, S., Stich, S., and Suresh,
  A.~T. (2020).
\newblock {SCAFFOLD}: Stochastic controlled averaging for federated learning.
\newblock In {\em International Conference on Machine Learning}.

\bibitem[Koloskova et~al., 2020]{Koloskova2020Decentralized}
Koloskova, A., Lin, T., Stich, S.~U., and Jaggi, M. (2020).
\newblock Decentralized deep learning with arbitrary communication compression.
\newblock In {\em International Conference on Learning Representations}.

\bibitem[Koloskova et~al., 2019]{koloskova2019decentralized}
Koloskova, A., Stich, S., and Jaggi, M. (2019).
\newblock Decentralized stochastic optimization and gossip algorithms with
  compressed communication.
\newblock In {\em International Conference on Machine Learning}.

\bibitem[Kovalev et~al., 2021]{kovalev2021linearly}
Kovalev, D., Koloskova, A., Jaggi, M., Richtarik, P., and Stich, S. (2021).
\newblock A linearly convergent algorithm for decentralized optimization:
  Sending less bits for free!
\newblock In {\em International Conference on Artificial Intelligence and
  Statistics}.

\bibitem[Krizhevsky, 2009]{Krizhevsky09learningmultiple}
Krizhevsky, A. (2009).
\newblock Learning multiple layers of features from tiny images.
\newblock Technical report.

\bibitem[LeCun et~al., 1998]{lecun1998gradientbased}
LeCun, Y., Bottou, L., Bengio, Y., and Haffner, P. (1998).
\newblock Gradient-based learning applied to document recognition.
\newblock In {\em IEEE}.

\bibitem[Li et~al., 2019]{li2019decentralized}
Li, Z., Shi, W., and Yan, M. (2019).
\newblock A decentralized proximal-gradient method with network independent
  step-sizes and separated convergence rates.
\newblock In {\em IEEE Transactions on Signal Processing}.

\bibitem[Lian et~al., 2017]{lian2017can}
Lian, X., Zhang, C., Zhang, H., Hsieh, C.-J., Zhang, W., and Liu, J. (2017).
\newblock Can decentralized algorithms outperform centralized algorithms? a
  case study for decentralized parallel stochastic gradient descent.
\newblock In {\em Advances in Neural Information Processing Systems}.

\bibitem[Liu et~al., 2021]{liu2021linear}
Liu, X., Li, Y., Wang, R., Tang, J., and Yan, M. (2021).
\newblock Linear convergent decentralized optimization with compression.
\newblock In {\em International Conference on Learning Representations}.

\bibitem[Lorenzo and Scutari, 2016]{lorenzo2016next}
Lorenzo, P.~D. and Scutari, G. (2016).
\newblock {NEXT:} in-network nonconvex optimization.
\newblock In {\em IEEE Transactions on Signal and Information Processing over
  Networks}.

\bibitem[Lu and De~Sa, 2020]{lu2020monniqua}
Lu, Y. and De~Sa, C. (2020).
\newblock Moniqua: Modulo quantized communication in decentralized {SGD}.
\newblock In {\em International Conference on Machine Learning}.

\bibitem[Niwa et~al., 2020]{niwa2020edge}
Niwa, K., Harada, N., Zhang, G., and Kleijn, W.~B. (2020).
\newblock Edge-consensus learning: Deep learning on p2p networks with
  nonhomogeneous data.
\newblock In {\em International Conference on Knowledge Discovery and Data
  Mining}.

\bibitem[Niwa et~al., 2021]{niwa2021asynchronous}
Niwa, K., Zhang, G., Kleijn, W.~B., Harada, N., Sawada, H., and Fujino, A.
  (2021).
\newblock Asynchronous decentralized optimization with implicit stochastic
  variance reduction.
\newblock In {\em International Conference on Machine Learning}.

\bibitem[Peaceman and Rachford, 1955]{peaceman1955numerical}
Peaceman, D.~W. and Rachford, H.~H. (1955).
\newblock The numerical solution of parabolic and elliptic differential
  equations.
\newblock {\em Journal of the Society for Industrial and Applied Mathematics}.

\bibitem[Rockafellar, 2015]{rockafellar2015convex}
Rockafellar, R.~T. (2015).
\newblock {\em Convex analysis}.
\newblock Princeton University Press.

\bibitem[Ryu and Boyd, 2015]{ryu2015primer}
Ryu, E.~K. and Boyd, S.~P. (2015).
\newblock A primer on monotone operator methods.
\newblock In {\em Applied and Computational Mathematics}.

\bibitem[Sherson et~al., 2019]{sherson2019derivation}
Sherson, T.~W., Heusdens, R., and Kleijn, W.~B. (2019).
\newblock Derivation and analysis of the primal-dual method of multipliers
  based on monotone operator theory.
\newblock In {\em IEEE Transactions on Signal and Information Processing over
  Networks}.

\bibitem[Stich et~al., 2018]{stich2018sparsified}
Stich, S.~U., Cordonnier, J.-B., and Jaggi, M. (2018).
\newblock Sparsified sgd with memory.
\newblock In {\em Advances in Neural Information Processing Systems}.

\bibitem[Tang et~al., 2018a]{tang2018communication}
Tang, H., Gan, S., Zhang, C., Zhang, T., and Liu, J. (2018a).
\newblock Communication compression for decentralized training.
\newblock In {\em Advances in Neural Information Processing Systems}.

\bibitem[Tang et~al., 2018b]{tang2018d2}
Tang, H., Lian, X., Yan, M., Zhang, C., and Liu, J. (2018b).
\newblock {$D^2$}: Decentralized training over decentralized data.
\newblock In {\em International Conference on Machine Learning}.

\bibitem[Vaswani et~al., 2017]{vaswani2017attention}
Vaswani, A., Shazeer, N., Parmar, N., Uszkoreit, J., Jones, L., Gomez, A.~N.,
  Kaiser, L.~u., and Polosukhin, I. (2017).
\newblock Attention is all you need.
\newblock In {\em Advances in Neural Information Processing Systems}.

\bibitem[Vogels et~al., 2021]{vogels2021relaysum}
Vogels, T., He, L., Koloskova, A., Karimireddy, S.~P., Lin, T., Stich, S.~U.,
  and Jaggi, M. (2021).
\newblock Relaysum for decentralized deep learning on heterogeneous data.
\newblock In {\em Advances in Neural Information Processing Systems}.

\bibitem[Vogels et~al., 2020]{vogels2020powergossip}
Vogels, T., Karimireddy, S.~P., and Jaggi, M. (2020).
\newblock Powergossip: Practical low-rank communication compression in
  decentralized deep learning.
\newblock In {\em Advances in Neural Information Processing Systems}.

\bibitem[Wu and He, 2018]{wu2018group}
Wu, Y. and He, K. (2018).
\newblock Group normalization.
\newblock In {\em European Conference on Computer Vision}.

\bibitem[Xiao et~al., 2017]{xiao2017/online}
Xiao, H., Rasul, K., and Vollgraf, R. (2017).
\newblock Fashion-mnist: a novel image dataset for benchmarking machine
  learning algorithms.
\newblock In {\em arXiv}.

\bibitem[Xiao et~al., 2007]{xiao2007distributed}
Xiao, L., Boyd, S.~P., and Kim, S. (2007).
\newblock Distributed average consensus with least-mean-square deviation.
\newblock In {\em Journal of Parallel and Distributed Computing}.

\bibitem[Xin et~al., 2020]{xin2020variance}
Xin, R., Khan, U.~A., and Kar, S. (2020).
\newblock Variance-reduced decentralized stochastic optimization with
  accelerated convergence.
\newblock In {\em IEEE Transactions on Signal and Information Processing over
  Networks}.

\bibitem[Zhang and Heusdens, 2018]{zhang2018distributed}
Zhang, G. and Heusdens, R. (2018).
\newblock Distributed optimization using the primal-dual method of multipliers.
\newblock In {\em IEEE Transactions on Signal and Information Processing over
  Networks}.

\end{thebibliography}

\newpage
\appendix
\section{Preliminary}
\label{sec:preliminary}
In this section, we introduce the definitions used in the following section
and briefly introduce the Douglas-Rachford splitting.
(See \cite{bauschke2011convex,ryu2015primer} for more details.)

\subsection{Definition}
\begin{definition}[Smooth Function]
Let $f: \mathcal{H}\rightarrow \mathbb{R} \cup \{ \infty \}$ be a closed, proper, and convex function.
$f$ is $L$-smooth if $f$ satisfies the following:
\begin{align*}
    f (x) \leq f (y) + \langle \nabla f(y), x - y \rangle + \frac{L}{2} \left\| x - y \right\|^2
    \qquad (\forall x, y \in \mathcal{H}).    
\end{align*}
\end{definition}

\begin{definition}[Strongly Convex Function]
Let $f: \mathcal{H}\rightarrow \mathbb{R} \cup \{ \infty \}$ be a closed, proper, and convex function.
$f$ is $\mu$-strongly convex if $f$ satisfies the following:
\begin{align*}
    f (x) \geq f (y) + \langle \nabla f(y), x - y \rangle + \frac{\mu}{2} \left\| x - y \right\|^2
    \qquad (\forall x, y \in \mathcal{H}).    
\end{align*}
\end{definition}

\begin{definition}[Conjugate Function]
Let $f: \mathcal{H}\rightarrow \mathbb{R} \cup \{ \infty \}$.
The conjugate function of $f$ is defined as follows:
\begin{align*}
    f^\ast(y) = \sup_{x\in\mathcal{H}} (\langle y, x \rangle - f(x)).
\end{align*}
\end{definition}

\begin{definition}[Nonexpansive Operator]
Let $D$ be a non-empty subset of $\mathcal{H}$.
An operator $T: D\rightarrow \mathcal{H}$ is nonexpansive if $T$ is $1$-Lipschitz continuous.
\end{definition}

\begin{definition}[Contractive Operator]
Let $D$ be a non-empty subset of $\mathcal{H}$.
An operator $T: D\rightarrow \mathcal{H}$ is $\beta$-contractive if $T$ is Lipschitz continuous with constant $\beta \in [0, 1)$.
\end{definition}

\begin{definition}[Proximal Operator]
Let $f: \mathcal{H} \rightarrow \mathbb{R} \cup \{ \infty \}$ be a closed, proper, and convex function.
The proximal operator of $f$ is defined as follows:
\begin{align*}
    \text{prox}_{f} (\mathbf{x}) = \text{argmin}_{\mathbf{y}} \{ f(\mathbf{y}) + \frac{1}{2} \| \mathbf{y} - \mathbf{x} \|^2 \}.
\end{align*}
\end{definition}

\begin{definition}[Resolvent]
Let $A: \mathcal{H}\rightarrow 2^{\mathcal{H}}$, 
the resolvent of $A$ is defined as follows:
\begin{align*}
    J_A = (\text{Id} + A)^{-1}.
\end{align*}
\end{definition}

\begin{definition}[Reflected Resolvent]
Let $A: \mathcal{H}\rightarrow 2^{\mathcal{H}}$, 
the reflected resolvent of $A$ is defined as follows:
\begin{align*}
    R_A = 2 J_A - \text{Id}.
\end{align*}
\end{definition}

\subsection{Douglas-Rachford Splitting}
In this section, we briefly introduce the Douglas-Rachford splitting \citep{douglas1955numerical}.

The Douglas-Rachford splitting can be applied to the following problem:
\begin{align}
    \inf_{\mathbf{x}} f(\mathbf{x}) + g(\mathbf{x}),
\end{align}
where $f$ and $g$ are closed, proper, and convex functions.
Let $\mathbf{x}^\star$ be the optimal solution of the above problem, 
the optimality condition can be written as follows:
\begin{align*}
    \mathbf{0} \in \partial f(\mathbf{x}^\star) + \partial g(\mathbf{x}^\star).
\end{align*}
From \citep[Proposition 26.1]{bauschke2011convex}, the optimality condition above is equivalent to the following:
\begin{align*}
    \mathbf{x}^\star =  J_{\alpha \partial f} (\text{Fix} (R_{\alpha \partial g} R_{\alpha \partial f})),
\end{align*}
where $\text{Fix}(R_{\alpha \partial g} R_{\alpha \partial f}) = \{ \mathbf{z} | R_{\alpha \partial g} R_{\alpha \partial f} \mathbf{z} = \mathbf{z} \}$
and $\alpha>0$.
The Douglas-Rachford splitting then computes the fixed point $\bar{\mathbf{z}} \in \text{Fix}(R_{\alpha \partial g} R_{\alpha \partial f})$ as follows:
\begin{align}
\label{eq:dr_splitting}
    \mathbf{z}^{(r+1)} &= ((1 - \theta) \text{Id} + \theta R_{\alpha \partial g} R_{\alpha \partial f}) \mathbf{z}^{(r)},
\end{align}
where $\theta \in (0, 1]$ is the hyperparameter of the Douglas-Rachford splitting.
Under certain assumptions, $((1 - \theta) \text{Id} + \theta R_{\alpha \partial g} R_{\alpha \partial f})$ is contractive \citep{giselsson2017linear},
and the update formula above is guaranteed to converge at the fixed point $\bar{\mathbf{z}}$.
Then, after converging to the fixed point $\bar{\mathbf{z}}$, the Douglas-Rachford splitting obtains the optimal solution $\mathbf{x}^\star$ as follows:
\begin{align}
    \mathbf{x}^{\star} &= J_{\alpha \partial f} \bar{\mathbf{z}}.
\end{align}
Moreover, by the definition of the reflected resolvent, 
the Douglas-Rachford splitting of Eq. \eqref{eq:dr_splitting} can be rewritten as follows:
\begin{align}
    \label{eq:explanation_of_dr_1}
    \mathbf{x}^{(r+1)} &= J_{\alpha \partial f} \mathbf{z}^{(r)}, \\
    \label{eq:explanation_of_dr_2}
    \mathbf{y}^{(r+1)} &= 2\mathbf{x}^{(r+1)} - \mathbf{z}^{(r)}, \\
    \label{eq:explanation_of_dr_3}
    \mathbf{z}^{(r+1)} &= (1-\theta) \mathbf{z}^{(r)} + \theta R_{\alpha \partial g} \mathbf{y}^{(r+1)}.
\end{align}
\newpage
\section{Derivation of Update Formulas of ECL}
\label{sec:derivation_of_ecl}
In this section, we briefly describe the derivation of the update formulas of the ECL.
(See \citep{niwa2020edge,niwa2021asynchronous,sherson2019derivation} for more details.)

\subsection{Derivation of Dual Problem}
First, we derive the Lagrangian function.
The Lagrangian function of Eq. \eqref{eq:ecl_primal} can be derived as follows:
\begin{align*}
    &\sum_{i\in\mathcal{V}} f_i (\mathbf{w}_i) - \sum_{i\in\mathcal{V}} \sum_{j\in\mathcal{N}_i} \langle \boldsymbol{\lambda}_{ij}, \mathbf{A}_{i|j} \mathbf{w}_i + \mathbf{A}_{j|i} \mathbf{w}_j \rangle \\
    &= \sum_{i\in\mathcal{V}} f_i (\mathbf{w}_i) - \sum_{i\in\mathcal{V}} \sum_{j\in\mathcal{N}_i} \langle \boldsymbol{\lambda}_{ij}, \mathbf{A}_{i|j} \mathbf{w}_i \rangle - \sum_{i\in\mathcal{V}} \sum_{j\in\mathcal{N}_i} \langle \boldsymbol{\lambda}_{ij}, \mathbf{A}_{j|i} \mathbf{w}_j \rangle \\
    &= \sum_{i\in\mathcal{V}} f_i (\mathbf{w}_i) - \sum_{i\in\mathcal{V}} \sum_{j\in\mathcal{N}_i} \langle \boldsymbol{\lambda}_{ij}+\boldsymbol{\lambda}_{ji}, \mathbf{A}_{i|j} \mathbf{w}_i \rangle,
\end{align*}
where $\boldsymbol{\lambda}_{ij}\in\mathbb{R}^d$ is the dual variable. 
Defining $\boldsymbol{\lambda}_{i|j}\coloneqq\boldsymbol{\lambda}_{ij} + \boldsymbol{\lambda}_{ji}$, 
the Lagrangian function can be rewritten as follows:
\begin{align*}
    \sum_{i\in\mathcal{V}} f_i (\mathbf{w}_i) - \sum_{i\in\mathcal{V}} \sum_{j\in\mathcal{N}_i} \langle \boldsymbol{\lambda}_{i|j}, \mathbf{A}_{i|j} \mathbf{w}_i \rangle.
\end{align*}
Note that the definition of $\boldsymbol{\lambda}_{i|j}$ implies that $\boldsymbol{\lambda}_{i|j} = \boldsymbol{\lambda}_{j|i}$.
Let $N$ be the number of nodes $|\mathcal{V}|$.
Here, $\mathcal{N}_i(j)$ denotes the $j$-th smallest index of the node in $\mathcal{N}_i$. 
We define $\mathbf{w}\in\mathbb{R}^{dN}$, $\mathbf{A}\in\mathbb{R}^{dN \times 2d|\mathcal{E}|}$ and $\boldsymbol{\lambda}\in\mathbb{R}^{2d|\mathcal{E}|}$ as follows:
\begin{align*} 
    \mathbf{w} &= 
    \begin{pmatrix}
        \mathbf{w}_1^\top, & \cdots, \mathbf{w}_N^\top 
    \end{pmatrix}^\top, \\
    \mathbf{A}_i &=
    \begin{pmatrix}
        \mathbf{A}_{i|\mathcal{N}_i(1)}, & \cdots, & \mathbf{A}_{i|\mathcal{N}_i(|\mathcal{N}_i|)} 
    \end{pmatrix}, \\
    \mathbf{A} &= \text{diag} (\mathbf{A}_1, \cdots, \mathbf{A}_N), \\
    \boldsymbol{\lambda}_i &=
    \begin{pmatrix}
        \boldsymbol{\lambda}_{i|\mathcal{N}_i(1)}^\top, & \cdots, & \boldsymbol{\lambda}_{i|\mathcal{N}_i(|\mathcal{N}_i|)}^\top 
    \end{pmatrix}^\top, \\
    \boldsymbol{\lambda} &=
    \begin{pmatrix}
        \boldsymbol{\lambda}_{1}^\top, & \cdots, & \boldsymbol{\lambda}_{N}^\top 
    \end{pmatrix}^\top.
\end{align*}
We define the function $f$ as follows:
\begin{align*}
    f(\mathbf{w}) = \sum_{i \in \mathcal{V}} f_i (\mathbf{w}_i).
\end{align*}
The Lagrangian function can be rewritten as follows:
\begin{align*}
    f(\mathbf{w}) - \langle \boldsymbol{\lambda}, \mathbf{A}^\top \mathbf{w} \rangle. 
\end{align*}
Then, the primal and dual problems can be defined as follows:
\begin{align}
    \label{eq:lagrangian_primel}
    \inf_{\mathbf{w}} \sup_{\boldsymbol{\lambda}} f(\mathbf{w}) - \langle \boldsymbol{\lambda}, \mathbf{A}^\top \mathbf{w} \rangle - \iota(\boldsymbol{\lambda})
    &= \inf_{\mathbf{w}} f(\mathbf{w}) + \iota^\ast (-\mathbf{A}^\top \mathbf{w}), \\
    \label{eq:lagrangian_dual}
    \sup_{\boldsymbol{\lambda}} \inf_{\mathbf{w}} f(\mathbf{w}) - \langle \boldsymbol{\lambda}, \mathbf{A}^\top \mathbf{w} \rangle - \iota(\boldsymbol{\lambda}) 
    &= - \inf_{\boldsymbol{\lambda}} f^\ast (\mathbf{A} \boldsymbol{\lambda}) + \iota(\boldsymbol{\lambda}),
\end{align}
where $\iota$ is the indicator function defined as follows:
\begin{align*}
    \iota (\boldsymbol{\lambda}) = \begin{cases}
    0 & \text{if} \; \boldsymbol{\lambda}_{i|j} = \boldsymbol{\lambda}_{j|i}, \; (\forall (i, j) \in \mathcal{E}) \\
    \infty & \text{otherwise}
    \end{cases}.
\end{align*}

\subsection{Derivation of Update Formulas}
Next, we derive the update formulas of the ECL.
By applying the Douglas-Rachford splitting of Eqs. (\ref{eq:explanation_of_dr_1}-\ref{eq:explanation_of_dr_3}) to the dual problem of Eq. \eqref{eq:lagrangian_dual},
we obtain the following update formulas:
\begin{align}
    \label{eq:pr_splitting_1}
    \boldsymbol{\lambda}^{(r+1)} &= J_{\alpha \mathbf{A}^\top \nabla f^\ast (\mathbf{A} \cdot )} \mathbf{z}^{(r)}, \\
    \label{eq:pr_splitting_2}
    \mathbf{y}^{(r+1)} &= 2 \boldsymbol{\lambda}^{(r+1)} - \mathbf{z}^{(r)}, \\
    \label{eq:pr_splitting_3}
    \mathbf{z}^{(r+1)} &= (1 - \theta) \mathbf{z}^{(r)} + \theta R_{\alpha \partial \iota} \mathbf{y}^{(r+1)},
\end{align}
where $\mathbf{z} \in \mathbb{R}^{2d|\mathcal{E}|}$ and $\mathbf{y}\in\mathbb{R}^{2d|\mathcal{E}|}$ can be decomposed into $\mathbf{z}_{i|j}\in\mathbb{R}^d$ and $\mathbf{y}_{i|j} \in \mathbb{R}^d$ as follows:
\begin{align*}
    \mathbf{z}_i &= (\mathbf{z}_{i|\mathcal{N}_i(1)}^\top, \ldots, \mathbf{z}_{i|\mathcal{N}_i(|\mathcal{N}_i|)}^\top)^\top, \\
    \mathbf{z} &= (\mathbf{z}_1^\top, \ldots, \mathbf{z}_N^\top)^\top, \\
    \mathbf{y}_i &= (\mathbf{y}_{i|\mathcal{N}_i(1)}^\top, \ldots, \mathbf{y}_{i|\mathcal{N}_i(|\mathcal{N}_i|)}^\top)^\top, \\
    \mathbf{y} &= (\mathbf{y}_1^\top, \ldots, \mathbf{y}_N^\top)^\top.
\end{align*}

\textbf{Update formulas of Eqs. (\ref{eq:pr_splitting_1}-\ref{eq:pr_splitting_2}):}
By the definition of the resolvent $J_{\alpha \mathbf{A}^\top \nabla f^\ast (\mathbf{A} \cdot )}$, we obtain
\begin{align}
    &\boldsymbol{\lambda}^{(r+1)} + \alpha \mathbf{A}^\top \nabla f^\ast (\mathbf{A} \boldsymbol{\lambda}^{(r+1)}) = \mathbf{z}^{(r)}, \nonumber \\
    \label{eq:pr_splitting_1_1}
    &\boldsymbol{\lambda}^{(r+1)} = \mathbf{z}^{(r)} - \alpha \mathbf{A}^\top \nabla f^\ast (\mathbf{A} \boldsymbol{\lambda}^{(r+1)}). 
\end{align}
We define $\mathbf{w}^{(r+1)}$ as follows:
\begin{align}
\label{eq:definition_of_w}
    \mathbf{w}^{(r+1)} = \nabla f^\ast (\mathbf{A}\boldsymbol{\lambda}^{(r+1)}).
\end{align}
From the property of the convex conjugate function, we obtain
\begin{align}
\label{eq:property_of_conjugate}
    \mathbf{A} \boldsymbol{\lambda}^{(r+1)} = \nabla f (\mathbf{w}^{(r+1)}).
\end{align}
Substituting Eqs. (\ref{eq:pr_splitting_1_1}-\ref{eq:definition_of_w}) into Eq. \eqref{eq:property_of_conjugate}, we obtain 
\begin{align*}
    \mathbf{0} = \nabla f (\mathbf{w}^{(r+1)}) + \mathbf{A} (\alpha \mathbf{A}^\top \mathbf{w}^{(r+1)} - \mathbf{z}^{(r)}).
\end{align*}
We then obtain the update formula of $\mathbf{w}$ as follows:
\begin{align*}
    \mathbf{w}^{(r+1)} = \text{argmin} &\{ f (\mathbf{w}) + \frac{\alpha}{2} \| \mathbf{A}^\top \mathbf{w} - \frac{1}{\alpha} \mathbf{z}^{(r)} \|^2 \}.
\end{align*}
Substituting Eq. \eqref{eq:definition_of_w} into Eq. \eqref{eq:pr_splitting_1_1},
we obtain
\begin{align*}
    \boldsymbol{\lambda}^{(r+1)} = \mathbf{z}^{(r)} - \alpha \mathbf{A}^\top \mathbf{w}^{(r+1)}.
\end{align*}
Then, the update formula of Eq. \eqref{eq:pr_splitting_2} is written as follows:
\begin{align*}
    \mathbf{y}^{(r+1)} = \mathbf{z}^{(r)} - 2\alpha \mathbf{A}^\top \mathbf{w}^{(r+1)}.
\end{align*}

\textbf{Update formula of Eq. \eqref{eq:pr_splitting_3}:}
From \cite[Lemma V\hspace{-1pt}I.2]{sherson2019derivation}, the update formula of Eq. \eqref{eq:pr_splitting_3} is rewritten as follows:
\begin{align*}
    \mathbf{z}^{(r+1)} = (1 - \theta) \mathbf{z}^{(r)} + \theta \mathbf{P} \mathbf{y}^{(r+1)},
\end{align*}
where $\mathbf{P}$ denotes the permutation matrix transforming $\mathbf{y}_{i|j}$ into $\mathbf{y}_{j|i}$ for all $(i, j) \in \mathcal{E}$.

In summary, the update formulas of the ECL can be derived as follows:
\begin{align}
    \mathbf{w}^{(r+1)} &= \text{argmin}_{\mathbf{w}} \{ f(\mathbf{w})
    + \frac{\alpha}{2} {\left\| \mathbf{A}^\top \mathbf{w} - \frac{1}{\alpha} \mathbf{z}^{(r)} \right\|}^2 \}, \\
    \mathbf{y}^{(r+1)} &= \mathbf{z}^{(r)} - 2 \alpha \mathbf{A}^\top \mathbf{w}^{(r+1)}, \\
    \mathbf{z}^{(r+1)} &= (1 - \theta) \mathbf{z}^{(r)} + \theta \mathbf{P} \mathbf{y}^{(r+1)}.
\end{align}
Then, by rewriting $\mathbf{w}$, $\mathbf{z}$ and $\mathbf{y}$ with $\mathbf{w}_i$, $\mathbf{z}_{i|j}$ and $\mathbf{y}_{i|j}$ respectively,
we obtain the update formulas of Eqs. (\ref{eq:update_w}-\ref{eq:update_z}).

\newpage
\section{Convergence Analysis of C-ECL}
\label{sec:convergence_analysis}
From the derivation of the ECL shown in Sec. \ref{sec:derivation_of_ecl},
the update formulas of Eq. \eqref{eq:update_w}, Eq. \eqref{eq:update_y}, and Eq. \eqref{eq:compression_in_cecl} in the C-ECL can be rewritten as follows:\footnote{By the definition of the reflected resolvent, the update formulas of Eq. \eqref{eq:pr_splitting_1} and Eq. \eqref{eq:pr_splitting_2} are equivalent to that of Eq. \eqref{eq:first_reflected_resolvent}.}
\begin{align}
    \label{eq:first_reflected_resolvent}
    \mathbf{y}^{(r+1)} &= R_{\alpha \mathbf{A}^\top \nabla f^\ast (\mathbf{A} \cdot) } \mathbf{z}^{(r)}, \\
    \label{eq:second_reflected_resolvent}
    \mathbf{z}^{(r+1)} &= \mathbf{z}^{(r)} + \theta \; \text{comp} ( R_{\alpha \partial \iota} \mathbf{y}^{(r+1)} - \mathbf{z}^{(r)}).
\end{align}
To simplify the notation, we define $g \coloneqq f^\ast(\mathbf{A} \cdot )$.
Then, Eq. \eqref{eq:first_reflected_resolvent} is rewritten as follows:
\begin{align}
    \label{eq:reformulated_first_reflected_resolvent}
    \mathbf{y}^{(r+1)} &= R_{\alpha \nabla g } \mathbf{z}^{(r)}.
\end{align}
In the following,
we analyze the convergence rate of the update formulas of Eq. \eqref{eq:reformulated_first_reflected_resolvent} and Eq. \eqref{eq:second_reflected_resolvent}.

\begin{lemma}
\label{lemma:singular_values}
Under Assumption \ref{assumption:no_isolated_nodes},
the maximum singular value and the minimum singular value of $\mathbf{A}$ are $\sqrt{N_{\text{max}}}$ and $\sqrt{N_{\text{min}}}$ respectively.
\end{lemma}
\begin{proof}
By the definition of $\mathbf{A}$,
we have $\mathbf{A}\mathbf{A}^\top = \text{diag}(|\mathcal{N}_1|\mathbf{I}, \cdots, |\mathcal{N}_N| \mathbf{I})$.
This implies that $\mathbf{A}\mathbf{A}^\top$ is not only a block diagonal matrix, but also a diagonal matrix.
Then, the eigen values of $\mathbf{A}\mathbf{A}^\top$ are $|\mathcal{N}_1|, \ldots, |\mathcal{N}_N|$.
This concludes the proof.
\end{proof}
\begin{remark}
\label{remark:singular_values}
Under Assumption \ref{assumption:no_isolated_nodes},
the maximum singular value and the minimum singular value of $\mathbf{A}^\top$ are $\sqrt{N_{\text{max}}}$ and $\sqrt{N_{\text{min}}}$ respectively.
\end{remark}

\begin{lemma}
\label{lemma:smoothness_strongness_of_conjugate}
Under Assumptions \ref{assumption:convex}, \ref{assumption:smoothness_and_strong_convex}, and \ref{assumption:no_isolated_nodes},
$f^\ast (\mathbf{A} \cdot)$ is $( N_{\text{max}}/\mu)$-smooth and $( N_{\text{min}}/L)$-strongly convex. 
\end{lemma}
\begin{proof}
From Assumption \ref{assumption:smoothness_and_strong_convex},
$f^\ast$ is $(1/\mu)$-smooth and $(1/L)$-strongly convex.
Then, for any $\boldsymbol{\lambda}$ and $\boldsymbol{\lambda}^\prime$,
by the $(1/\mu)$-smoothness of $f^\ast$ and Lemma \ref{lemma:singular_values}, we have
\begin{align*}
    f^\ast (\mathbf{A}\boldsymbol{\lambda}) 
    &\leq f^\ast (\mathbf{A}\boldsymbol{\lambda}^\prime) + \langle \nabla f^\ast (\mathbf{A}\boldsymbol{\lambda}^\prime), \mathbf{A}\boldsymbol{\lambda} - \mathbf{A}\boldsymbol{\lambda}^\prime \rangle + \frac{1}{2\mu} \left\| \mathbf{A}\boldsymbol{\lambda} - \mathbf{A}\boldsymbol{\lambda}^\prime \right\|^2 \\
    &\leq f^\ast (\mathbf{A}\boldsymbol{\lambda}^\prime) + \langle \mathbf{A}^\top \nabla f^\ast (\mathbf{A}\boldsymbol{\lambda}^\prime), \boldsymbol{\lambda} - \boldsymbol{\lambda}^\prime \rangle + \frac{ N_{\text{max}}}{2\mu} \left\| \boldsymbol{\lambda} - \boldsymbol{\lambda}^\prime \right\|^2.
\end{align*}
Because $f^\ast$ is $(1/L)$-strongly convex, 
from Lemma \ref{lemma:singular_values}, for any $\boldsymbol{\lambda}$ and $\boldsymbol{\lambda}^\prime$, we have
\begin{align*}
    f^\ast (\mathbf{A}\boldsymbol{\lambda}) 
    &\geq f^\ast (\mathbf{A}\boldsymbol{\lambda}^\prime) + \langle \nabla f^\ast (\mathbf{A}\boldsymbol{\lambda}^\prime), \mathbf{A}\boldsymbol{\lambda} - \mathbf{A}\boldsymbol{\lambda}^\prime \rangle + \frac{1}{2L} \left\| \mathbf{A}\boldsymbol{\lambda} - \mathbf{A}\boldsymbol{\lambda}^\prime \right\|^2 \\
    &\geq f^\ast (\mathbf{A}\boldsymbol{\lambda}^\prime) + \langle \mathbf{A}^\top \nabla f^\ast (\mathbf{A}\boldsymbol{\lambda}^\prime), \boldsymbol{\lambda} - \boldsymbol{\lambda}^\prime \rangle + \frac{ N_{\text{min}}}{2L} \left\| \boldsymbol{\lambda} - \boldsymbol{\lambda}^\prime \right\|^2.
\end{align*}
This concludes the proof.
\end{proof}

We define $\delta$ as follows:
\begin{align*}
    \delta \coloneqq \text{max} \left( 
    \frac{\frac{\alpha  N_{\text{max}}}{\mu} - 1}{\frac{\alpha  N_{\text{max}}}{\mu} + 1},
    \frac{1 - \frac{\alpha  N_{\text{min}}}{L}}{1 + \frac{\alpha  N_{\text{min}}}{L}}
    \right).
\end{align*}
Note that when Assumptions \ref{assumption:convex}, \ref{assumption:smoothness_and_strong_convex}, and \ref{assumption:no_isolated_nodes} are satisfied, and when $\alpha>0$,
$0\leq \delta < 1$ is satisfied.

\begin{lemma}
\label{lemma:contractive}
Under Assumptions \ref{assumption:convex}, \ref{assumption:smoothness_and_strong_convex}, and \ref{assumption:no_isolated_nodes},
$R_{\alpha \nabla g}$ is $\delta$-contractive for any $\alpha \in (0, \infty)$.
\end{lemma}
\begin{proof}
The statement follows from \citep[Theorem 1]{giselsson2017linear} and Lemma \ref{lemma:smoothness_strongness_of_conjugate}.
\end{proof}

\begin{lemma}
\label{lemma:contractiveness_of_two_reflected_resolvent}
Under Assumptions \ref{assumption:convex}, \ref{assumption:smoothness_and_strong_convex}, and \ref{assumption:no_isolated_nodes},
$R_{\alpha \partial \iota} R_{\alpha \nabla g}$ is $\delta$-contractive for any $\alpha\in(0, \infty)$.
\end{lemma}
\begin{proof}
From \citep[Corollary 23.9]{bauschke2011convex} and \citep[Theorem 20.25]{bauschke2011convex},
$R_{\alpha \partial \iota}$ is nonexpansive for any $\alpha \in (0, \infty)$.
Then, from Lemma \ref{lemma:contractive}, 
for any $\mathbf{z}$ and $\mathbf{z}^\prime$, we have
\begin{align*}
    \| R_{\alpha \partial \iota} R_{\alpha \nabla g} \mathbf{z} - R_{\alpha \partial \iota} R_{\alpha \nabla g} \mathbf{z}^\prime \|
    &\leq \| R_{\alpha \nabla g} \mathbf{z} - R_{\alpha \nabla g} \mathbf{z}^\prime \| \\
    &\leq \delta \| \mathbf{z} - \mathbf{z}^\prime \|.
\end{align*}
This concludes the proof.
\end{proof}

In the following, we use the notation $\mathbb{E}_r[\cdot]$ as the expectation over the randomness in the round $r$.

\begin{lemma}
\label{lemma:contractive_of_z}
Let $\bar{\mathbf{z}} \in \text{Fix}(R_{\alpha \partial \iota} R_{\alpha \nabla g})$.
Under Assumptions \ref{assumption:compression}, \ref{assumption:convex}, \ref{assumption:smoothness_and_strong_convex}, and \ref{assumption:no_isolated_nodes},
$\mathbf{z}^{(r+1)}$ and $\mathbf{z}^{(r)}$ generated by Eqs. (\ref{eq:first_reflected_resolvent}-\ref{eq:second_reflected_resolvent}) satisfy the following:
\begin{align}
    \mathbb{E}_r \| \mathbf{z}^{(r+1)} - \bar{\mathbf{z}} \| 
    \leq \{ |1 - \theta| + \theta \delta + \sqrt{1 - \tau} (\theta + |1 - \theta| \delta + \delta ) \} \| \mathbf{z}^{(r)} - \bar{\mathbf{z}} \|. 
\end{align}
\end{lemma}
\begin{proof}
Combining Eq. \eqref{eq:second_reflected_resolvent} and Eq. \eqref{eq:reformulated_first_reflected_resolvent},
the update formula of $\mathbf{z}$ can be written as follows:
\begin{align*}
    \mathbf{z}^{(r+1)} = \mathbf{z}^{(r)} + \theta \text{comp} ((R_{\alpha \partial \iota} R_{\alpha \nabla g} - \text{Id}) \mathbf{z}^{(r)}).
\end{align*}
Let $\omega^{(r)}$ be the parameter used to compress $((R_{\alpha \partial \iota} R_{\alpha \nabla g} - \text{Id}) \mathbf{z}^{(r)})$.
In the following, to use Eq. \eqref{eq:linearlity_of_compression} and Eq. \eqref{eq:linearlity_of_compression_2} in Assumption \ref{assumption:compression},
we rewrite the update formula of $\mathbf{z}$ as follows:
\begin{align*}
    \mathbf{z}^{(r+1)} &= \mathbf{z}^{(r)} + \theta \text{comp} ((R_{\alpha \partial \iota} R_{\alpha \nabla g} - \text{Id}) \mathbf{z}^{(r)} ; \omega^{(r)}).
\end{align*}
Because $\bar{\mathbf{z}} \in \text{Fix}(R_{\alpha \partial \iota} R_{\alpha \nabla g})$, 
we have $(R_{\alpha \partial \iota} R_{\alpha \nabla g} - \text{Id}) \bar{\mathbf{z}} = \mathbf{0}$.
Under Assumption \ref{assumption:compression}, because $\textbf{comp}(\mathbf{0} ; \omega) = \mathbf{0}$ for any $\omega$, we have
\begin{align*}
    \bar{\mathbf{z}} &= \bar{\mathbf{z}} + \theta \text{comp} ((R_{\alpha \partial \iota} R_{\alpha \nabla g} - \text{Id}) \bar{\mathbf{z}} ; \omega^{(r)}).
\end{align*}
We then have
\begin{align*}
    &\mathbb{E}_r \| \mathbf{z}^{(r+1)} - \bar{\mathbf{z}} \| \\
    &= \mathbb{E}_r \| \mathbf{z}^{(r)} - \bar{\mathbf{z}} 
    +\theta \; \text{comp}( (R_{\alpha \partial \iota} R_{\alpha \nabla g} - \text{Id}) \mathbf{z}^{(r)} ; \omega^{(r)}) - \theta \; \text{comp}( (R_{\alpha \partial \iota} R_{\alpha \nabla g} - \text{Id}) \bar{\mathbf{z}} ; \omega^{(r)}) \| \\
    &\stackrel{(a)}{\leq} \| (1 - \theta) (\mathbf{z}^{(r)} - \bar{\mathbf{z}}) + \theta ( R_{\alpha \partial \iota} R_{\alpha \nabla g} \mathbf{z}^{(r)} - R_{\alpha \partial \iota} R_{\alpha \nabla g} \bar{\mathbf{z}} ) \| \\
    &\qquad + \mathbb{E}_r \| \theta ( \mathbf{z}^{(r)} - \bar{\mathbf{z}} - R_{\alpha \partial \iota} R_{\alpha \nabla g} \mathbf{z}^{(r)} + R_{\alpha \partial \iota} R_{\alpha \nabla g} \bar{\mathbf{z}} ) \\
    &\qquad\qquad - \theta \{ \text{comp} (\mathbf{z}^{(r)} - \bar{\mathbf{z}} 
    - R_{\alpha \partial \iota} R_{\alpha \nabla g} \mathbf{z}^{(r)} + R_{\alpha \partial \iota} R_{\alpha \nabla g} \bar{\mathbf{z}} ; \omega^{(r)}) \} \| \\
    &\stackrel{(b)}{\leq}  \| (1 - \theta) (\mathbf{z}^{(r)} - \bar{\mathbf{z}}) + \theta ( R_{\alpha \partial \iota} R_{\alpha \nabla g} \mathbf{z}^{(r)} - R_{\alpha \partial \iota} R_{\alpha \nabla g} \bar{\mathbf{z}} ) \| \\
    &\qquad + \sqrt{1 - \tau} \| \theta (\mathbf{z}^{(r)} - \bar{\mathbf{z}}) - \theta ( R_{\alpha \partial \iota} R_{\alpha \nabla g} \mathbf{z}^{(r)} - R_{\alpha \partial \iota} R_{\alpha \nabla g} \bar{\mathbf{z}} ) \| \\
    &\leq \underbrace{\| (1 - \theta) (\mathbf{z}^{(r)} - \bar{\mathbf{z}}) + \theta ( R_{\alpha \partial \iota} R_{\alpha \nabla g} \mathbf{z}^{(r)} - R_{\alpha \partial \iota} R_{\alpha \nabla g} \bar{\mathbf{z}} ) \|}_{T_1} \\
    &\qquad + \underbrace{\sqrt{1 - \tau} \| \theta (\mathbf{z}^{(r)} - \bar{\mathbf{z}}) + (1 - \theta) ( R_{\alpha \partial \iota} R_{\alpha \nabla g} \mathbf{z}^{(r)} - R_{\alpha \partial \iota} R_{\alpha \nabla g} \bar{\mathbf{z}} ) \|}_{T_2} \\
    &\qquad + \underbrace{\sqrt{1 - \tau} \| R_{\alpha \partial \iota} R_{\alpha \nabla g} \mathbf{z}^{(r)} - R_{\alpha \partial \iota} R_{\alpha \nabla g} \bar{\mathbf{z}} \|}_{T_3},
\end{align*}
where we use Assumption \ref{assumption:compression} for (a) and (b).
From Lemma \ref{lemma:contractiveness_of_two_reflected_resolvent}, 
$T_1$, $T_2$, and $T_3$ are upper bounded as follows:
\begin{align*}
    T_1 
    &\leq |1 - \theta| \| \mathbf{z}^{(r)} - \bar{\mathbf{z}} \| + \theta \| R_{\alpha \partial \iota} R_{\alpha \nabla g} \mathbf{z}^{(r)} - R_{\alpha \partial \iota} R_{\alpha \nabla g} \bar{\mathbf{z}}  \| \\
    &\leq (|1 - \theta| + \theta \delta) \| \mathbf{z}^{(r)} - \bar{\mathbf{z}} \|, \\
    T_2 
    &\leq \sqrt{1 - \tau} (\theta \| \mathbf{z}^{(r)} - \bar{\mathbf{z}} \| + |1 - \theta| \| R_{\alpha \partial \iota} R_{\alpha \nabla g} \mathbf{z}^{(r)} - R_{\alpha \partial \iota} R_{\alpha \nabla g} \bar{\mathbf{z}}  \| )\\
    &\leq \sqrt{1 - \tau} (\theta + |1 - \theta| \delta ) \| \mathbf{z}^{(r)} - \bar{\mathbf{z}} \|, \\
    T_3 &\leq \sqrt{1 - \tau} \delta \| \mathbf{z}^{(r)} - \bar{\mathbf{z}} \|.
\end{align*}
Therefore, we obtain
\begin{align*}
    \mathbb{E}_r \| \mathbf{z}^{(r+1)} - \bar{\mathbf{z}} \| 
    \leq \{ |1 - \theta| + \theta \delta + \sqrt{1 - \tau} (\theta + |1 - \theta| \delta + \delta ) \} \| \mathbf{z}^{(r)} - \bar{\mathbf{z}} \|. 
\end{align*}
This concludes the proof.
\end{proof}
\begin{lemma}
\label{lemma:condition_of_contractiveness_of_z}
Let $\bar{\mathbf{z}} \in \text{Fix}(R_{\alpha \partial \iota} R_{\alpha \nabla g})$.
Under Assumptions \ref{assumption:compression}, \ref{assumption:convex}, \ref{assumption:smoothness_and_strong_convex}, and \ref{assumption:no_isolated_nodes},
when $\tau \geq 1 - ( \frac{1 - \delta}{1 + \delta} )^2$ and $\theta$ satisfies the following:
\begin{align}
\label{eq:domain_of_theta}
    \theta \in \left( \frac{2\delta \sqrt{1 - \tau}}{(1 - \delta)(1 - \sqrt{1 - \tau})}, \frac{2}{(1+\delta)(1 + \sqrt{1 - \tau})} \right),
\end{align}
$\mathbf{z}^{(r+1)}$ and $\mathbf{z}^{(r)}$ generated by Eqs. (\ref{eq:first_reflected_resolvent}-\ref{eq:second_reflected_resolvent}) satisfy the following:
\begin{align}
    \mathbb{E}_r \| \mathbf{z}^{(r+1)} - \bar{\mathbf{z}} \| 
    \leq \{ |1 - \theta| + \theta \delta + \sqrt{1 - \tau} (\theta + |1 - \theta| \delta + \delta ) \} \| \mathbf{z}^{(r)} - \bar{\mathbf{z}} \| 
    < \| \mathbf{z}^{(r)} - \bar{\mathbf{z}} \|.
\end{align}
\end{lemma}
\begin{proof}
When $\tau \geq 1 - ( \frac{1 - \delta}{1 + \delta} )^2$,
the domain of Eq. \eqref{eq:domain_of_theta} is non-empty and contains $1$.
Then, when $\tau \geq 1 - ( \frac{1 - \delta}{1 + \delta} )^2$ and $\theta$ satisfies the following:
\begin{align*}
    \theta \in \left( \frac{2\delta \sqrt{1 - \tau}}{(1 - \delta)(1 - \sqrt{1 - \tau})}, 1 \right],
\end{align*}
we have
\begin{align*}
    &|1 - \theta| + \theta \delta + \sqrt{1 - \tau} (\theta + |1 - \theta| \delta + \delta )  \\
    &= 1 + 2\delta \sqrt{1-\tau} - \theta (1 - \sqrt{1-\tau})(1 - \delta) < 1. 
\end{align*}
When $\tau \geq 1 - ( \frac{1 - \delta}{1 + \delta} )^2$ and $\theta$ satisfies the following:
\begin{align*}
    \theta \in \left[ 1,  \frac{2}{(1+\delta)(1 + \sqrt{1 - \tau})} \right),
\end{align*}
we have
\begin{align*}
    &|1 - \theta| + \theta \delta + \sqrt{1 - \tau} (\theta + |1 - \theta| \delta + \delta )  \\
    &= - 1 + \theta (1 + \delta ) (1 + \sqrt{1 - \tau}) < 1. 
\end{align*}
This concludes the proof.
\end{proof}

In the following,
we define the operator $T$ as follows:
\begin{align*}
    T \mathbf{z}^{(r)} \coloneqq \text{argmin}_{\mathbf{w}} \{ f(\mathbf{w})
    + \frac{\alpha}{2} {\left\| \mathbf{A}^\top \mathbf{w} - \frac{1}{\alpha} \mathbf{z}^{(r)} \right\|}^2 \}.
\end{align*}

\begin{lemma}
\label{lemma:lipschitz_of_w}
Under Assumptions \ref{assumption:convex}, \ref{assumption:smoothness_and_strong_convex}, and \ref{assumption:no_isolated_nodes},
the operator $T$ is $(\sqrt{N_{\text{max}}} / (\mu + \alpha N_{\text{min}}))$-Lipschitz continuous.
\end{lemma}
\begin{proof}
We have
\begin{align*}
    &\text{argmin}_{\mathbf{w}} \{ f(\mathbf{w})
    + \frac{\alpha}{2} {\left\| \mathbf{A}^\top \mathbf{w} - \frac{1}{\alpha} \mathbf{z}^{(r)} \right\|}^2 \} \\
    &= \text{argmin}_{\mathbf{w}} \{ f(\mathbf{w}) + \frac{\alpha}{2} \| \mathbf{A}^\top \mathbf{w} \|^2
    - \langle \mathbf{w}, \mathbf{A}\mathbf{z}^{(r)} \rangle \} \\
    &= \text{argmin}_{\mathbf{w}} \{ f(\mathbf{w}) + \frac{\alpha}{2} \| \mathbf{A}^\top \mathbf{w} \|^2
    - \frac{1}{2} \| \mathbf{w} \|^2 + \frac{1}{2} \| \mathbf{w} - \mathbf{A} \mathbf{z}^{(r)} \|^2 \} \\
    &= \text{prox}_{f + \frac{\alpha}{2} \| \mathbf{A}^\top \cdot \|^2
    - \frac{1}{2} \| \cdot \|^2} (\mathbf{A} \mathbf{z}^{(r)}).
\end{align*}
From \citep[Example 23.3]{bauschke2011convex}, 
the proximal operator of $(f + \frac{\alpha}{2} \| \mathbf{A}^\top \cdot \|^2 - \frac{1}{2} \| \cdot \|^2)$
is the resolvent of $\nabla (f + \frac{\alpha}{2} \| \mathbf{A}^\top \cdot \|^2 - \frac{1}{2} \| \cdot \|^2)$,
and the resolvent can be rewritten as follows:
\begin{align*}
    J_{\nabla (f + \frac{\alpha}{2} \| \mathbf{A}^\top \cdot \|^2 - \frac{1}{2} \| \cdot \|^2)} 
    &= (\text{Id} + \nabla (f + \frac{\alpha}{2} \| \mathbf{A}^\top \cdot \|^2 - \frac{1}{2} \| \cdot \|^2))^{-1} \\
    &= (\nabla (f + \frac{\alpha}{2} \| \mathbf{A}^\top \cdot \|^2))^{-1} \\
    &= \nabla (f + \frac{\alpha}{2} \| \mathbf{A}^\top \cdot \|^2)^\ast.
\end{align*}
From Assumption \ref{assumption:smoothness_and_strong_convex} and Remark \ref{remark:singular_values}, 
$(f + \frac{\alpha}{2} \| \mathbf{A}^\top \cdot \|^2)$ is $(\mu + \alpha N_{\text{min}})$-strongly convex.
Then, $(f + \frac{\alpha}{2} \| \mathbf{A}^\top \cdot \|^2)^\ast$ is $(1/(\mu + \alpha N_{\text{min}}))$-smooth.
Then, the proximal operator of $(f + \frac{\alpha}{2} \| \mathbf{A}^\top \cdot \|^2 - \frac{1}{2} \| \cdot \|^2)$ is $(1/(\mu + \alpha N_{\text{min}}))$-Lipschitz continuous.
Then, we have
\begin{align*}
    \left\| \text{prox}_{f + \frac{\alpha}{2} \| \mathbf{A}^\top \cdot \|^2
    - \frac{1}{2} \| \cdot \|^2} (\mathbf{A} \mathbf{z}) - \text{prox}_{f + \frac{\alpha}{2} \| \mathbf{A}^\top \cdot \|^2
    - \frac{1}{2} \| \cdot \|^2} (\mathbf{A} \mathbf{z}^\prime) \right\| 
    &\leq \frac{1}{\mu + \alpha N_{\text{min}}} \| \mathbf{A} \mathbf{z} - \mathbf{A} \mathbf{z}^\prime \| \\
    &\stackrel{(a)}{\leq} \frac{\sqrt{N_{\text{max}}}}{\mu + \alpha N_{\text{min}}}  \| \mathbf{z} - \mathbf{z}^\prime \|,
\end{align*}
where we use Lemma \ref{lemma:singular_values} for (a).
This concludes the proof.
\end{proof}

\begin{lemma}
\label{lemma:relationship_between_primal_and_dual}
Suppose that Assumptions \ref{assumption:convex}, \ref{assumption:smoothness_and_strong_convex}, and \ref{assumption:no_isolated_nodes} hold.
Let $\boldsymbol{\lambda}^\star$ be the optimal solution of the dual problem of Eq. \eqref{eq:lagrangian_dual}.
Then, $\nabla f^\ast (\mathbf{A} \boldsymbol{\lambda}^\star)$ is the optimal solution of the primal problem of Eq. \eqref{eq:lagrangian_primel}.
\end{lemma}
\begin{proof}
By using the optimality condition \citep[Theorem 27.1]{rockafellar2015convex} of $\boldsymbol{\lambda}^\star$,
we have
\begin{align*}
    - \mathbf{A}^\top \nabla f^\ast (\mathbf{A} \boldsymbol{\lambda}^\ast) \in \partial \iota (\boldsymbol{\lambda}^\star).
\end{align*}
From the property of the convex conjugate function of $\iota$, we get
\begin{align*}
     \boldsymbol{\lambda}^\star \in \partial \iota^\ast (- \mathbf{A}^\top \nabla f^\ast (\mathbf{A} \boldsymbol{\lambda}^\star)).
\end{align*}
Then, by multiplying $\mathbf{A}$ and using the property of the convex conjugate function of $f$, we have
\begin{align*}
     \nabla f ( \nabla f^\ast ( \mathbf{A} \boldsymbol{\lambda}^\star)) \in \mathbf{A} \partial \iota^\ast (- \mathbf{A}^\top \nabla f^\ast (\mathbf{A} \boldsymbol{\lambda}^\star)).
\end{align*}
From the above equation, $\nabla f^\ast (\mathbf{A} \boldsymbol{\lambda}^\star)$ satisfies the optimality condition of Eq. \eqref{eq:lagrangian_primel}.
This concludes the proof.
\end{proof}

\begin{lemma}
\label{lemma:w_and_z}
Let $\bar{\mathbf{z}} \in \text{Fix}(R_{\alpha \partial \iota} R_{\alpha \nabla g})$,
and let $\mathbf{w}^\star$ be the optimal solution of Eq. \eqref{eq:ecl_primal}.
Under Assumptions \ref{assumption:compression}, \ref{assumption:convex}, \ref{assumption:smoothness_and_strong_convex}, and \ref{assumption:no_isolated_nodes},
$\mathbf{w}^{(r+1)}$ and $\mathbf{z}^{(r-1)}$ generated by Alg. \ref{alg:C-ECL} satisfy the following:
\begin{align}
    &\mathbb{E}_{r-1} \|\mathbf{w}^{(r+1)} - \mathbf{w}^\star \| \nonumber \\
    &\leq \frac{\sqrt{N_{\text{max}}}}{\mu + \alpha N_{\text{min}}} \{ |1 - \theta| + \theta \delta + \sqrt{1 - \tau} (\theta + |1 - \theta| \delta + \delta ) \} \| \mathbf{z}^{(r-1)} - \bar{\mathbf{z}} \|.
\end{align}
\end{lemma}
\begin{proof}
Let $\boldsymbol{\lambda}^\star$ be the optimal solution of the dual problem of Eq. \eqref{eq:lagrangian_dual}.
We have $\boldsymbol{\lambda}^\star = J_{\alpha \nabla g} \bar{\mathbf{z}}$ because $\bar{\mathbf{z}} \in \text{Fix}(R_{\alpha \partial \iota} R_{\alpha \nabla g})$.
By Lemma \ref{lemma:relationship_between_primal_and_dual} and the definition of $\mathbf{w}$ of Eq. \eqref{eq:definition_of_w},
$\mathbf{w}^\star$ can be obtained from the fixed point $\bar{\mathbf{z}}$ as $\mathbf{w}^\star = T \bar{\mathbf{z}}$.
Then, we have
\begin{align*}
    &\mathbb{E}_{r-1} \|\mathbf{w}^{(r+1)} - \mathbf{w}^\star \| \\
    &\leq \mathbb{E}_{r-1} \| T \mathbf{z}^{(r)} - T \bar{\mathbf{z}} \| \\
    &\stackrel{(a)}{\leq} \frac{\sqrt{N_{\text{max}}}}{\mu + \alpha N_{\text{min}}} \mathbb{E}_{r-1} \| \mathbf{z}^{(r)} - \bar{\mathbf{z}} \| \\
    &\stackrel{(b)}{\leq} \frac{\sqrt{N_{\text{max}}}}{\mu + \alpha N_{\text{min}}} \{ |1 - \theta| + \theta \delta + \sqrt{1 - \tau} (\theta + |1 - \theta| \delta + \delta ) \} \| \mathbf{z}^{(r-1)} - \bar{\mathbf{z}} \|,
\end{align*}
where we use Lemma \ref{lemma:lipschitz_of_w} for (a) and use Lemma \ref{lemma:contractive_of_z} for (b).
This concludes the proof.
\end{proof}
\begin{lemma}
\label{lemma:upper_bound_of_w}
Let $\bar{\mathbf{z}} \in \text{Fix}(R_{\alpha \partial \iota} R_{\alpha \nabla g})$,
and let $\mathbf{w}^\star$ be the optimal solution of Eq. \eqref{eq:ecl_primal}.
Under Assumptions \ref{assumption:compression}, \ref{assumption:convex}, \ref{assumption:smoothness_and_strong_convex}, and \ref{assumption:no_isolated_nodes},
$\mathbf{w}^{(r+1)}$ generated by Alg. \ref{alg:C-ECL} satisfies the following:
\begin{align}
    &\mathbb{E} \|\mathbf{w}^{(r+1)} - \mathbf{w}^\star \| \nonumber \\
    &\leq \frac{\sqrt{N_{\text{max}}}}{\mu + \alpha N_{\text{min}}} \left\{ |1 - \theta| + \theta \delta + \sqrt{1 - \tau} (\theta + |1 - \theta| \delta + \delta ) \right\}^r \| \mathbf{z}^{(0)} - \bar{\mathbf{z}} \|.
\end{align}
\end{lemma}
\begin{proof}
The statement follows from Lemma \ref{lemma:contractive_of_z} and Lemma \ref{lemma:w_and_z}.
\end{proof}

\setcounter{theorem}{0}
\begin{theorem}
Let $\bar{\mathbf{z}} \in \text{Fix}(R_{\alpha \partial \iota} R_{\alpha \nabla g})$.
Under Assumptions \ref{assumption:compression}, \ref{assumption:convex}, \ref{assumption:smoothness_and_strong_convex}, and \ref{assumption:no_isolated_nodes},
when $\tau \geq 1 - ( \frac{1 - \delta}{1 + \delta} )^2$ and $\theta$ satisfies the following:
\begin{align}
    \theta \in \left( \frac{2\delta \sqrt{1 - \tau}}{(1 - \delta)(1 - \sqrt{1 - \tau})}, \frac{2}{(1+\delta)(1 + \sqrt{1 - \tau})} \right),
\end{align}
$\mathbf{w}^{(r+1)}$ generated by Alg. \ref{alg:C-ECL} linearly converges to the optimal solution $\mathbf{w}^\star$ as follows:
\begin{align}
    &\mathbb{E} \|\mathbf{w}^{(r+1)} - \mathbf{w}^\star \| \nonumber \\
    &\leq \frac{\sqrt{N_{\text{max}}}}{\mu + \alpha N_{\text{min}}} \left\{ |1 - \theta| + \theta \delta + \sqrt{1 - \tau} (\theta + |1 - \theta| \delta + \delta ) \right\}^r \| \mathbf{z}^{(0)} - \bar{\mathbf{z}} \|.
\end{align}
\end{theorem}
\begin{proof}
The statement follows from Lemma \ref{lemma:condition_of_contractiveness_of_z} and Lemma \ref{lemma:upper_bound_of_w}.
\end{proof}
\setcounter{corollary}{0}
\begin{corollary}
Let $\bar{\mathbf{z}} \in \text{Fix}(R_{\alpha \partial \iota} R_{\alpha \nabla g})$.
Under Assumptions \ref{assumption:compression}, \ref{assumption:convex}, \ref{assumption:smoothness_and_strong_convex}, and \ref{assumption:no_isolated_nodes},
when $\tau = 1$ and $\theta$ satisfies the following:
\begin{align}
    \theta \in \left( 0, \frac{2}{1+\delta} \right),
\end{align}
$\mathbf{w}^{(r+1)}$ generated by Alg. \ref{alg:C-ECL} linearly converges to the optimal solution $\mathbf{w}^\star$ as follows:
\begin{align}
    \mathbb{E} \|\mathbf{w}^{(r+1)} - \mathbf{w}^\star \| \nonumber \leq \frac{\sqrt{N_{\text{max}}}}{\mu + \alpha N_{\text{min}}} \left\{ |1 - \theta| + \theta \delta \right\}^r \| \mathbf{z}^{(0)} - \bar{\mathbf{z}} \|.
\end{align}
\end{corollary}
\begin{proof}
The statement follows from Theorem \ref{theorem:main_convergence_analysis}.
\end{proof}
\begin{corollary}
Under Assumptions \ref{assumption:compression}, \ref{assumption:convex}, \ref{assumption:smoothness_and_strong_convex}, and \ref{assumption:no_isolated_nodes},
when $\tau \geq 1 - ( \frac{1 - \delta}{1 + \delta} )^2$,
the optimal convergence rate of Eq. \eqref{eq:convergence_rate_of_cecl} in the C-ECL is achieved when $\theta=1$.
\end{corollary}
\begin{proof}
By Theorem \ref{theorem:main_convergence_analysis},
when $\theta\leq 1$, we have
\begin{align*}
    \mathbb{E} \|\mathbf{w}^{(r+1)} - \mathbf{w}^\star \| 
    \leq \frac{\sqrt{N_{\text{max}}}}{\mu + \alpha N_{\text{min}}} \{ 1 + 2\delta \sqrt{1-\tau} - \theta (1 - \sqrt{1-\tau})(1 - \delta) \}^r \| \mathbf{z}^{(0)} - \bar{\mathbf{z}} \|.
\end{align*}
Because $(1 - \sqrt{1-\tau})(1 - \delta) > 0$, $\{ 1 + 2\delta \sqrt{1-\tau} - \theta (1 - \sqrt{1-\tau})(1 - \delta) \}$ decreases when $\theta$ approaches $1$.

When $\theta \geq 1$, we have
\begin{align*}
    \mathbb{E} \|\mathbf{w}^{(r+1)} - \mathbf{w}^\star \| 
    \leq \frac{\sqrt{N_{\text{max}}}}{\mu + \alpha N_{\text{min}}} \{ - 1 + \theta (1 + \delta ) (1 + \sqrt{1 - \tau}) \}^r \| \mathbf{z}^{(0)} - \bar{\mathbf{z}} \|.
\end{align*}
Because $(1 + \delta ) (1 + \sqrt{1 - \tau}) > 0$, $\{ - 1 + \theta (1 + \delta ) (1 + \sqrt{1 - \tau}) \}$ decreases when $\theta$ approaches $1$.
Therefore, the optimal convergence rate is achieved when $\theta=1$.
\end{proof}
\begin{corollary}
Under Assumptions \ref{assumption:compression}, \ref{assumption:convex}, \ref{assumption:smoothness_and_strong_convex}, and \ref{assumption:no_isolated_nodes},
when $\tau=1$,
the optimal convergence rate of Eq. \eqref{eq:convergence_rate_of_ecl} is achieved when $\theta=1$.
\end{corollary}
\begin{proof}
The statement follows from Corollary \ref{corollary:optimal_convergence_rate}.
\end{proof}

\newpage
\section{Experimental Setting}
\label{sec:experimental_setting}

\subsection{Hyperparameter}
In this section, we describe the detailed hyperparameters used in our experiments.

\textbf{FashionMNIST:}
We set the learning rate to $0.001$, batch size to $100$, and number of epochs to $1500$.
To avoid the overfitting, we use $\textbf{RandomCrop}$ of PyTorch for the data augmentation.
In the ECL, following the previous work \citep{niwa2021asynchronous}, we set $\alpha$ as follows:
\begin{align}
\label{eq:alpha_ecl}
    \alpha=\frac{1}{\eta |\mathcal{N}_i| (K-1)},
\end{align}
where $K$ is the number of local steps.
In the C-ECL, when we use $\textbf{rand}_{k\%}$ as the compression operator,
the number of local steps can be regarded to be $\frac{100 K}{k}$.
Then, in the C-ECL, we set $\alpha$ as follows:
\begin{align}
\label{eq:alpha_cecl}
    \alpha = \frac{1}{\eta |\mathcal{N}_i| (\frac{100 K}{k}-1)}.
\end{align}
Note that $\alpha$ is set to the different values between the nodes by the definition of $\alpha$ in Eqs. (\ref{eq:alpha_ecl}-\ref{eq:alpha_cecl}).
For the D-PSGD and the PowerGossip, we use Metropolis-Hastings weights \citep{xiao2007distributed} as the weights of the edges.

\textbf{CIFAR10:}
We set the learning rate to $0.005$, batch size to $100$, and number of epochs to $2500$.
To avoid the overfitting, we use $\textbf{RandomCrop}$ and $\textbf{RandomHorizontalFlip}$ of PyTorch for the data augmentation.
For the ECL and the C-ECL, we set $\alpha$ in the same way as the FashionMNIST.
For the D-PSGD and the PowerGossip, we use Metropolis–Hastings weights as the weights of the edges.

\subsection{Network Topology}
Fig. \ref{fig:visualization_of_graph} shows the visualization of the network topology used in our experiments.
\begin{figure}[H]
\center
\subfigure[Chain]{
    \includegraphics[width=0.5\hsize]{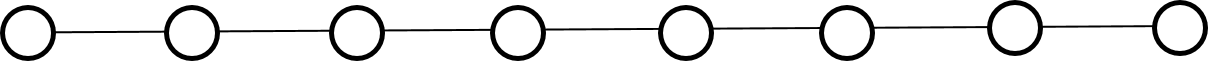}
}
\\
\subfigure[Ring]{
    \includegraphics[width=0.3\hsize]{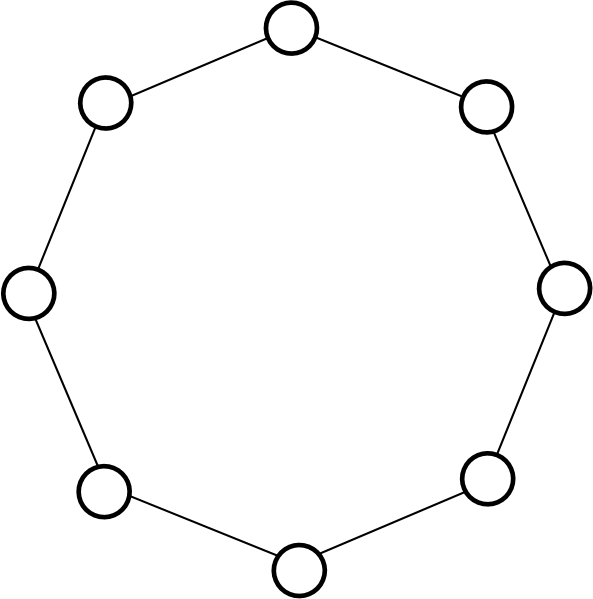}
}
\subfigure[Multiplex Ring]{
    \includegraphics[width=0.3\hsize]{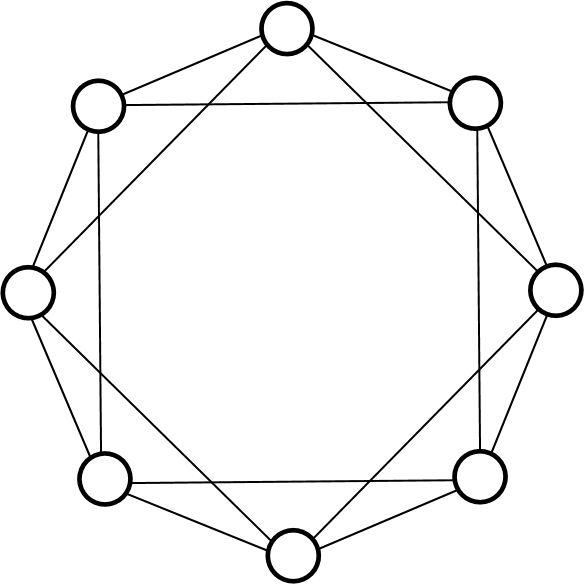}
}
\subfigure[Fully Connected Graph]{
    \includegraphics[width=0.3\hsize]{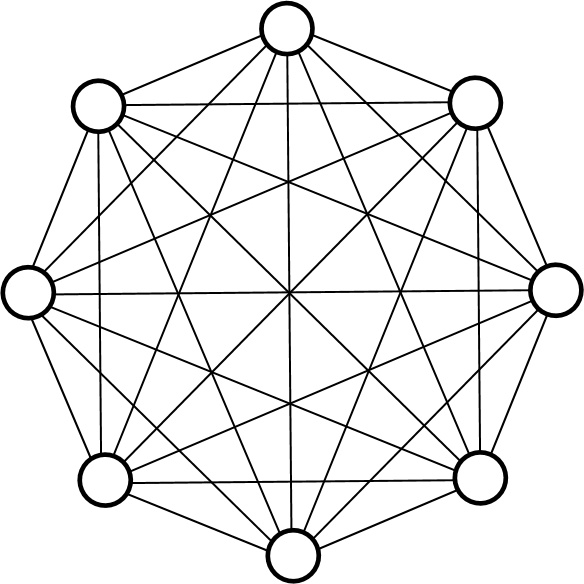}
}
\caption{Visualization of the network topology.}
\label{fig:visualization_of_graph}
\end{figure}

\end{document}